\documentclass[preprint,12pt]{elsarticle}

\usepackage{hyperref}
\usepackage{booktabs}
\usepackage{multirow}
\usepackage{array}
\usepackage{xcolor}
\usepackage{amsthm}
\usepackage{algorithm}
\usepackage{algpseudocode}
\newtheorem{theorem}{Theorem}[section]
\newtheorem{corollary}{Corollary}[section]
\newtheorem*{remark}{Remark}
\usepackage{amssymb}
\usepackage{etoolbox}
\usepackage{amsmath}

\begin{document}

\begin{frontmatter}



\title{Optimal-Transport-Guided Functional Flow Matching for Turbulent Field Generation in Hilbert Space}


\author[label1,label2]{Kunpeng Li}
\author[label1]{Chenguang Wan}
\author[label1]{Zhisong Qu}
\author[label1]{Kyungtak Lim}
\author[label3]{Virginie Grandgirard}
\author[label1,label3]{Xavier Garbet}
\author[label5]{Hua Yu}
\author[label2,label4]{Ong Yew Soon}

\affiliation[label1]{organization={School of Physical and Mathematical Sciences, Nanyang Technological University},
            city={,},
            postcode={637371},
            country={Singapore}}

\affiliation[label2]{organization={College of Computing and Data Science, Nanyang Technological University},
            city={Singapore},
            postcode={639798},
            country={Singapore}}

\affiliation[label3]{organization={CEA, IRFM},
            postcode={F-13108},
            state={Saint Paul-lez-Durance},
            country={France}}

\affiliation[label4]{organization={Centre for Frontier AI Research, Agency for Science, Technology and Research},
            city={Singapore},
            postcode={138648},
            country={Singapore}}

\affiliation[label5]{organization={Dalian Jiaotong University},
            city={Dalian},
            postcode={116028},
            country={China}}

\begin{abstract}
High-fidelity modeling of turbulent flows requires capturing complex spatiotemporal dynamics and multi-scale intermittency, posing a fundamental challenge for traditional knowledge-based systems. While deep generative models, such as diffusion models and Flow Matching, have shown promising performance, they are fundamentally constrained by their discrete, pixel-based nature. This limitation restricts their applicability in turbulence computing, where data inherently exists in a functional form. To address this gap, we propose Functional Optimal Transport Conditional Flow Matching (FOT-CFM), a generative framework defined directly in infinite-dimensional function space.  
Unlike conventional approaches defined on fixed grids, FOT-CFM treats physical fields as elements of an infinite-dimensional Hilbert space, and learns resolution-invariant generative dynamics directly at the level of probability measures.
By integrating Optimal Transport (OT) theory, we construct deterministic, straight-line probability paths between noise and data measures in Hilbert space. This formulation enables simulation-free training and significantly accelerates the sampling process. We rigorously evaluate the proposed system on a diverse suite of chaotic dynamical systems, including the Navier-Stokes equations, Kolmogorov Flow, and Hasegawa-Wakatani equations, all of which exhibit rich multi-scale turbulent structures. Experimental results demonstrate that FOT-CFM achieves superior fidelity in reproducing high-order turbulent statistics and energy spectra compared to state-of-the-art baselines.
\end{abstract}

\begin{graphicalabstract}

\centering
\includegraphics[width=1\textwidth]{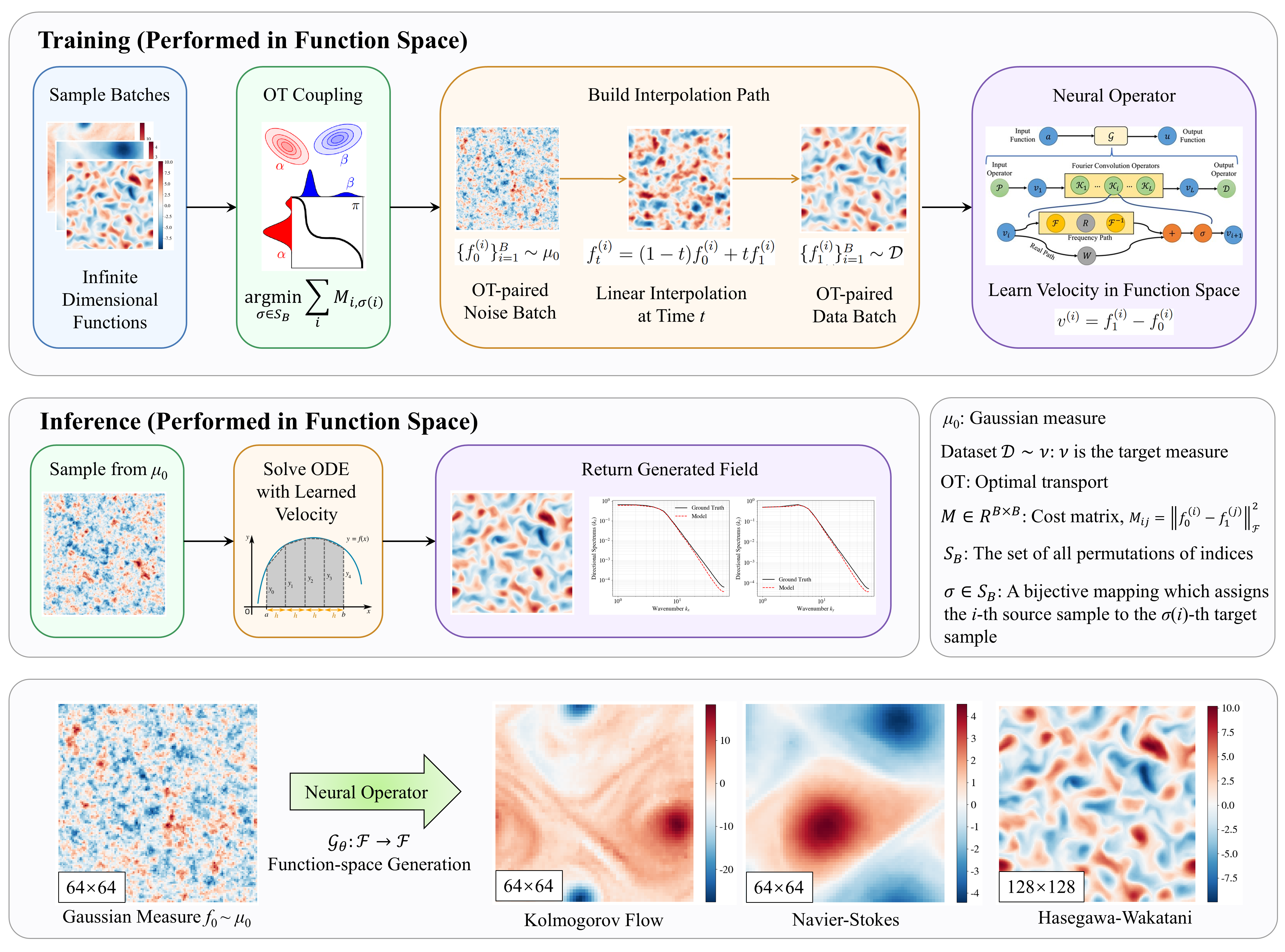}
Function-space OT alignment enables fast and high-fidelity turbulence generation.

\end{graphicalabstract}

\begin{highlights}
\item We generalize Conditional Flow Matching (CFM) from finite-dimensional Euclidean spaces to infinite-dimensional Hilbert spaces. Specifically, we formulate conditional-to-marginal path mixing directly at the level of probability measures and weak continuity equations, which avoids density-based constructions that are not natural in infinite dimensions. We further prove that the aggregated conditional vector field in function space induces the correct marginal probability path, and establish the equivalence between the conditional and marginal training objectives (up to a parameter-independent constant).

\item We incorporate Optimal Transport (OT) theory into functional CFM to construct OT-guided straight-line probability paths between the source (noise) and target (data) measures. By enforcing transport-aligned trajectories, FOT-CFM rectifies the generative flow and reduces trajectory curvature. Combined with the simulation-free CFM training objective, this yields high-quality sampling with significantly fewer NFE than diffusion-based or curved ODE-based baselines.

\item By parameterizing the vector field with Neural Operators, FOT-CFM inherently learns the continuous physical operator independent of the discretization mesh, enabling zero-shot super-resolution. Practical benchmarks on complex chaotic systems, including Navier-Stokes, Kolmogorov Flow, and Hasegawa-Wakatani equations, demonstrate that our method accurately reproduces high-order turbulent statistics and energy spectra, while achieving a significant reduction in inference latency compared with baseline methods.

\end{highlights}


\end{frontmatter}



\section{Introduction}\label{sec1}
Turbulent flows are ubiquitous in both natural and engineering systems, spanning atmospheric circulation and ocean currents to aerodynamic design and combustion processes \cite{pope2001turbulent}. Understanding and modeling turbulence is essential for climate prediction \cite{hussain2012evaluation}, energy technologies \cite{conway2008turbulence, fouladi2020wind}, and industrial fluid dynamics \cite{wang2024recent}. However, achieving high-fidelity turbulence modeling remains a fundamental challenge in scientific computing and knowledge-based systems, due to the complex spatiotemporal dynamics and pronounced multiscale structure of turbulent flows. Motivated by the high cost of direct numerical simulation and the growing demand for fast surrogate generation, generative models (GMs) have recently attracted increasing attention for turbulence modeling \cite{drygala2022generative, drygala2024comparison}. Nevertheless, a fundamental representation mismatch remains: each sample in turbulence is more naturally described as a physical field over a spatial domain, that is, as a function rather than as a finite-dimensional vector or tensor defined on fixed discretizations. This function-valued nature is not well aligned with most existing generative modeling frameworks, which are predominantly formulated in finite-dimensional Euclidean spaces (e.g., vectors in $\mathbb{R}^n$).

Although generative models have achieved impressive performance across a wide range of domains, including images \cite{kim2026multi, dhariwal2021diffusion, kang2023scaling}, 3D data \cite{gao2022get3d, achlioptas2018learning}, audio \cite{zhao2025tia2v, oord2016wavenet, vasquez2019melnet}, and video \cite{video1, video2}, with increasing adoption in machine learning security \cite{security1, security2}, natural language processing \cite{NL1,NL2}, protein design \cite{protein1, protein2}, and physics and engineering problems \cite{chen2021fast, liu2025cgan, chen2026hydraulic}, their underlying discrete parameterizations are not well suited to scientific settings, where consistency across resolutions and computational meshes is often essential. 

Similar function-valued data arises broadly in PDE-governed applications such as seismology, geophysics, oceanography, aerodynamic vehicle design, and weather forecasting \cite{yang2021seismicwavepropagationinversion, Wen}. Functional representations are also standard in 3D vision and graphics, where scenes may be parameterized as radiance fields \cite{mildenhall2021nerf} or signed distance functions \cite{park2019deepsdf}. These observations motivate generative modeling frameworks defined directly in infinite-dimensional function spaces.

Substantial progress has been made in adapting generative models to infinite-dimensional spaces \cite{dupont2022data, li2025functional, zhang2025flow}. A pivotal development is Denoising Diffusion Operator (DDO) \cite{ddo}. DDO defines the score operator using the Fréchet derivative of the log-density with respect to a reference Gaussian measure (rather than the translation-invariant Lebesgue measure used in finite dimensions). To approximate this score in practice, DDO generalizes the denoising score matching objective \cite{song2020score} to Hilbert spaces. Sampling is then performed by reversing the diffusion process via infinite-dimensional Langevin dynamics using the learned score operator.

In parallel, flow-based generative modeling \cite{lipman2022flow} has been extended to function spaces through the Functional Flow Matching (FFM) \cite{ffm}, which considers a Gaussian noise corruption process in Hilbert space. FFM constructs a path of conditional Gaussian measures that approximately interpolates between a fixed reference Gaussian measure and a given function. By marginalizing these conditional paths over the data distribution, a path of measures connecting the noise and data distributions is obtained. This construction establishes couplings between source and target samples that implicitly correspond to an optimal transport map between Gaussians in the Euclidean setting.

Notwithstanding these theoretical strides, developing an efficient and generalized flow-based framework for functions remains impeded by two major technical challenges:

First, while pioneering works have demonstrated the feasibility of generative modeling directly in Hilbert spaces, existing function-space generative frameworks still lack a unified and rigorous conditional–marginal consistency theory in the infinite-dimensional setting. In particular, density-based marginalization arguments commonly used in finite-dimensional Euclidean spaces do not extend straightforwardly to Hilbert spaces, several key questions remain unresolved: whether conditional path mixing is well-defined at the level of probability measures, whether the aggregated conditional vector field induces the correct marginal probability path, and whether the tractable conditional training objective is equivalent to the ideal marginal objective.

Second, geometric and dynamical choices in flow-path design can translate into high computational cost at inference time. DDO \cite{ddo} generation process relies on many iterative denoising steps (e.g., annealed Langevin dynamics or numerical SDE solvers) to produce high-quality samples. The sampling procedure of FFM \cite{ffm} via numerical ODE integration still incurs substantial computational cost when the induced flow is difficult to integrate accurately. More fundamentally, existing functional frameworks do not explicitly enforce a globally optimal transport geometry between the source and target measures, which can lead to poorly aligned, high-curvature characteristic flows.

\begin{figure}[h]
\centering
\includegraphics[width=1\textwidth]{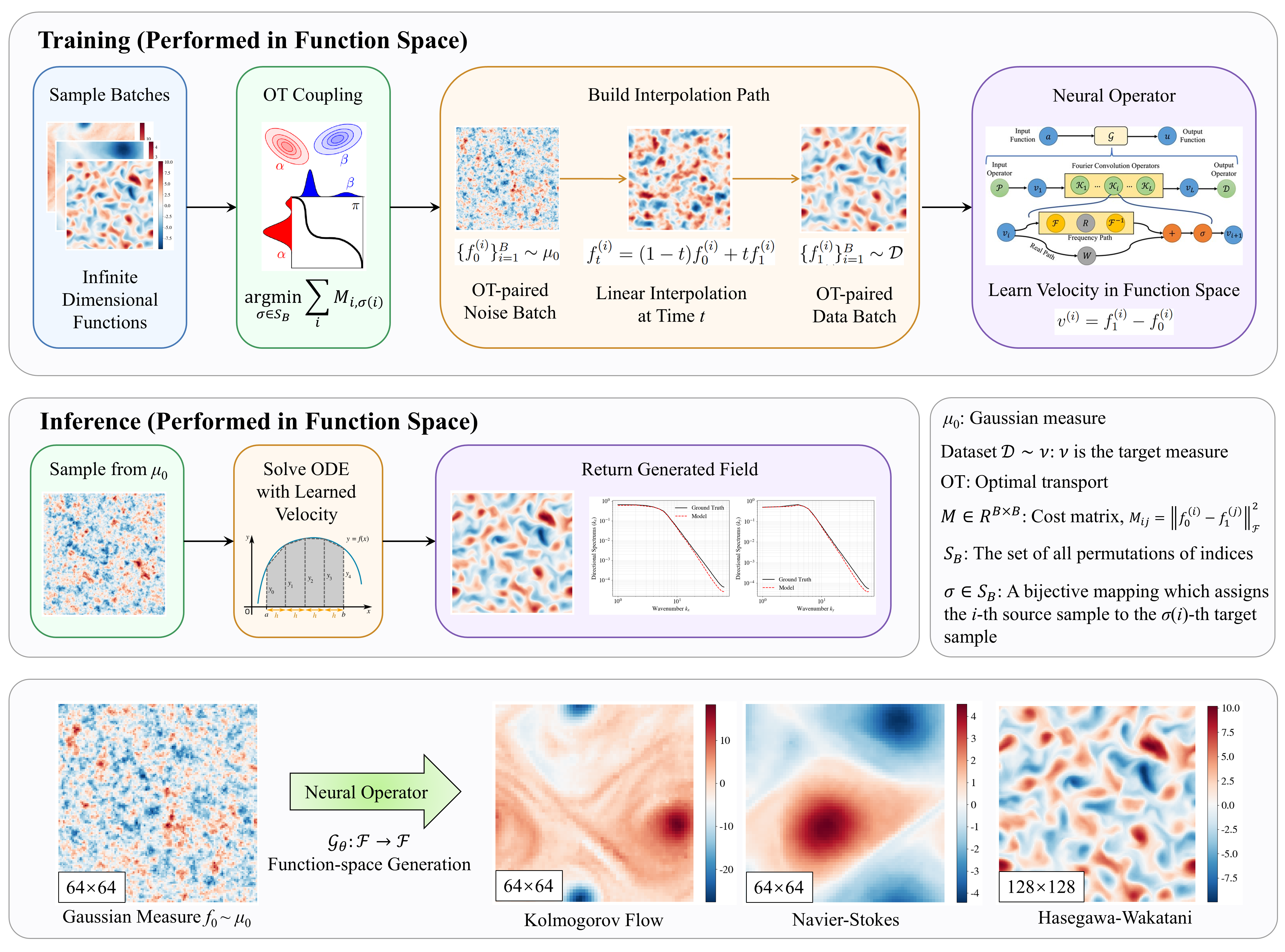}
\caption{FOT-CFM in infinite-dimensional function space, OT-aligned operator training and ODE Sampling}\label{fig:framework}
\end{figure}

To address these limitations, we propose Functional Optimal Transport Conditional Flow Matching (FOT-CFM), a unifying framework (shown as Fig. \ref{fig:framework}) for efficient and resolution-invariant generative modeling in Hilbert space. Our main contributions are summarized as follows:

(1) We generalize Conditional Flow Matching (CFM) from finite-dimensional Euclidean spaces to infinite-dimensional Hilbert spaces. Specifically, to address the first challenge, we formulate conditional-to-marginal path mixing directly at the level of probability measures and weak continuity equations, which avoids density-based constructions that are not natural in infinite dimensions. We further prove that the aggregated conditional vector field in function space induces the correct marginal probability path, and establish the equivalence between the conditional and marginal training objectives (up to a parameter-independent constant).

(2) Aiming at the second challenge, we incorporate Optimal Transport (OT) theory \cite{villani2008optimal} into functional CFM to construct OT-guided straight-line probability paths between the source (noise) and target (data) measures. By enforcing transport-aligned trajectories, FOT-CFM rectifies the generative flow and reduces trajectory curvature. Combined with the simulation-free CFM training objective, this yields high-quality sampling with significantly fewer NFE than diffusion-based or curved ODE-based baselines.

(3) By parameterizing the vector field with Neural Operators, FOT-CFM inherently learns the continuous physical operator independent of the discretization mesh, enabling zero-shot super-resolution. Practical benchmarks on complex chaotic systems, including Navier-Stokes, Kolmogorov Flow, and Hasegawa-Wakatani equations, demonstrate that our method accurately reproduces high-order turbulent statistics and energy spectra, while achieving a significant reduction in inference latency compared with baseline methods.

The rest of the paper is arranged as follows: We first introduce the theoretical background and terminology in Section 2. Section 3 formally presents the methodology of the FOT-CFM. Section 4 is dedicated to empirical validation, where we benchmark the proposed method against competitive baselines across multiple chaotic flow scenarios. Finally, Section 5 provides concluding remarks and directions for future work.

\section{Background and Terminology}\label{sec2}

\subsection{Functional Flow Matching}\label{subsec1}
Functional Flow Matching (FFM) \cite{ffm} extends classical flow matching from finite-dimensional Euclidean spaces to
infinite-dimensional function spaces. Let $(\mathcal{F},\langle\cdot,\cdot\rangle_{\mathcal{F}})$ be a separable Hilbert
space of functions with Borel $\sigma$-algebra $\mathcal{B}(\mathcal{F})$.
Let the reference measure be a Gaussian measure
$\mu_0=\mathcal{N}(m_0,C_0)$ on $\mathcal{F}$, with mean $m_0\in\mathcal{F}$ and covariance operator $C_0:\mathcal{F}\to\mathcal{F}$.
FFM learns a time-dependent velocity field
$u:[0,1]\times\mathcal{F}\to\mathcal{F}$ that transports $\mu_0$ to a target distribution $\mu_1=\nu$
through a continuous path of measures $(\mu_t)_{t\in[0,1]}$ satisfying the weak continuity equation:
\begin{equation}
\int_{0}^{1}\!\int_{\mathcal{F}}
\Big(\partial_t \psi(g,t)+\big\langle u_t(g),\nabla_g\psi(g,t)\big\rangle_{\mathcal{F}}\Big)\,d\mu_t(g)\,dt =0,
\qquad 
\label{eq:ffm_continuity}
\end{equation}
for all appropriate test functions $\psi:\mathcal{F}\times[0,1]\to\mathbb{R}$, and $\mu_{t=0}=\mu_0,\ \mu_{t=1}=\mu_1$.
Sampling $f_0\sim\mu_0$, a generated function is obtained by integrating the function-space ODE

\begin{equation}
\frac{d f_t}{dt}=u(t,f_t),\qquad f_{t=0}=f_0,
\label{eq:ffm_ode}
\end{equation}
whose terminal state satisfies $f_1\sim\nu$.

For a given velocity field $u_t(\cdot)$, define the associated flow maps $\phi_t:\mathcal{F}\to\mathcal{F}$ by
$f_t=\phi_t(f_0)$, where $\phi_t$ satisfies the functional differential equation

\begin{equation}
\frac{\partial}{\partial t}\phi_t = u_t\circ \phi_t,
\qquad
\phi_0 = \mathrm{Id}_{\mathcal{F}},
\label{eq:ffm_flow}
\end{equation}
with $\mathrm{Id}_{\mathcal{F}}$ the identity operator on $\mathcal{F}$.
The measure path can be generated by pushforward:
$\mu_t=(\phi_t)_{\#}\mu_0$.

\paragraph{Conditional paths and marginalization}
The marginal (global) velocity field needed for the standard regression objective is typically intractable in function spaces.
FFM therefore introduces a conditional velocity $u_t^f$ conditioned on a target function $f\sim \nu$,
together with conditional paths $(\mu_t^f)_{t\in[0,1]}$ that interpolate between $\mu_0$ and an $f$-centered measure $\mu_1^f$.
Marginalizing these conditionals yields the global path and velocity:
\begin{equation}
\begin{aligned}
\mu_t(A)
&=\int_{\mathcal{F}} \mu_t^f(A)\,d\nu(f),\\
u_t(g)
&=\int_{\mathcal{F}} u_t^f(g)\,\frac{d\mu_t^f}{d\mu_t}(g)\,d\nu(f).
\end{aligned}
\label{eq:ffm_marginalization}
\end{equation}
for any $A\in\mathcal{B}(\mathcal{F})$, where $\frac{d\mu_t^f}{d\mu_t}$ is the Radon--Nikodym derivative.

\paragraph{Gaussian conditional path (closed form)}
In practice, the conditional paths are often chosen to be Gaussian:
\[
\mu_t^f = \mathcal{N}\!\big(m_t^f,(\sigma_t^f)^2 C_0\big),
\qquad
m_t^f = t f,
\qquad
\sigma_t^f = 1-(1-\sigma_{\min})t,
\]
with a small $\sigma_{\min}>0$.
Then the conditional flow and conditional velocity admit closed forms:
\begin{equation}
\begin{aligned}
\phi_t^f(f_0)=\sigma_t^f f_0 + m_t^f
& = \bigl(1-(1-\sigma_{\min})t\bigr) f_0 + t f, \\
\qquad
u_t^f(g)=\frac{\dot{\sigma}_t^f}{\sigma_t^f}\bigl(g-m_t^f\bigr)+\dot m_t^f
& =\frac{1-\sigma_{\min}}{1-(1-\sigma_{\min})t}\,(t f-g)+f.
\end{aligned}
\label{eq:ffm_closed_form}
\end{equation}
Although the theory requires $\sigma_{\min}>0$, setting $\sigma_{\min}=0$ is often used in practice without adverse effects.

\paragraph{Training objective}
The model $u_\theta(t,g)$ is trained via the conditional regression loss
\begin{equation}
\mathcal{L}_c(\theta)
=\mathbb{E}_{t,f,\;g\sim \mu_t^f}
\Big[\big\|u_t^f(g)-u_\theta(t,g)\big\|_{\mathcal{F}}^2\Big],
\label{eq:ffm_conditional_loss}
\end{equation}
which can be shown to be equivalent to the (intractable) marginal loss up to an additive constant.

\subsection{Optimal Transport}
\label{subsec:optimal_transport}

The static optimal transport (OT) problem seeks a transport plan that moves mass from one probability measure to another with minimal effort.
In the context of generative modeling on function spaces, we are particularly interested in the 2-Wasserstein distance between the source (noise) measure $\mu_0$ and the target (data) measure $\mu_1$ defined on the separable Hilbert space $\mathcal{F}$.
Consider the quadratic cost function $c(x, y) = \| x - y \|_{\mathcal{F}}^2$, which measures the squared Hilbert-space norm between two functions $x, y \in \mathcal{F}$.
The squared 2-Wasserstein distance is defined as the solution to the Kantorovich minimization problem:
\begin{equation}
    W_2^2(\mu_0, \mu_1) = \inf_{\pi \in \Pi(\mu_0, \mu_1)} \int_{\mathcal{F} \times \mathcal{F}} \| x - y \|_{\mathcal{F}}^2 \, d\pi(x, y),
    \label{eq:static_ot}
\end{equation}
where $\Pi(\mu_0, \mu_1)$ denotes the set of all joint probability measures (couplings) on $\mathcal{F} \times \mathcal{F}$ whose marginals are $\mu_0$ and $\mu_1$, respectively.
Under mild conditions (e.g., probability measures with finite second moments), a solution to Eq.~\eqref{eq:static_ot} exists \cite{villani2008optimal}, and $W_2$ defines a metric on the space of probability distributions over $\mathcal{F}$.
Crucially, the optimal coupling $\pi^*$ typically concentrates on a deterministic map (Monge map) that pushes $\mu_0$ to $\mu_1$ along geodesic paths, which in our Hilbert space setting corresponds to straight-line trajectories minimizing the kinetic energy of the flow.

While the static formulation (Eq.~\eqref{eq:static_ot}) focuses on the optimal coupling, the dynamic formulation of OT connects directly to generative flows.
The Benamou-Brenier formula \cite{benamou2000computational} establishes that the squared Wasserstein distance $W_2^2(\mu_0, \mu_1)$ is equivalent to the minimal kinetic energy required to transport mass from $\mu_0$ to $\mu_1$:
\begin{equation}
    W_2^2(\mu_0, \mu_1) = \inf_{(\mu_t, v_t)} \int_0^1 \int_{\mathcal{F}} \| v_t(x) \|_{\mathcal{F}}^2 \, d\mu_t(x) \, dt,
    \label{eq:dynamic_ot}
\end{equation}
subject to the continuity equation $\partial_t \mu_t + \nabla \cdot (v_t \mu_t) = 0$ with boundary conditions $\mu_0, \mu_1$.
The pair $(\mu_t, v_t)$ achieving this infimum defines the Wasserstein geodesic connecting $\mu_0$ and $\mu_1$.
In the Euclidean (and Hilbert) setting with the quadratic cost, this geodesic corresponds to the displacement interpolation \cite{mccann1997convexity}, where mass moves along straight lines with constant speed.
Specifically, if $\pi^*$ is the optimal coupling from the static problem, the geodesic path is given by the law of $x_t = (1-t)x_0 + t x_1$ for $(x_0, x_1) \sim \pi^*$.
Consequently, the vector field $v_t$ generating this path minimizes the transport cost and results in straight trajectories, which is the ideal target for our training objective.

\section{Methodology of the FOT-CFM}\label{sec3}

This section builds a complete pipeline from measure-theoretic foundations to practical algorithms for function-space generative modeling. Section \ref{subsec:mixture_paths} starts by formulating a mixture of conditional probability paths directly at the level of probability measures and the weak continuity equation, since density-based constructions are generally ill-defined in infinite-dimensional Hilbert spaces due to the absence of a translation-invariant Lebesgue measure. It then establishes the conditional-to-marginal consistency through rigorous results, proving that the aggregated conditional vector field induces the correct marginal probability path. Section \ref{subsec:objective} moves from path construction to learning and introduces the Functional Conditional Flow Matching (FCFM) objective. It shows that the tractable conditional objective is equivalent to the ideal marginal objective up to a parameter-independent constant, and therefore yields the same gradient, enabling efficient stochastic training by sampling. Building on this theoretical feasibility, Section \ref{subsec:FOT-CFM} addresses the issue of training efficiency by incorporating optimal transport techniques, replacing independent coupling with OT-aligned pairings and displacement interpolation to obtain straighter trajectories and lower-NFE sampling. Finally, Section \ref{subsec:algorithm} turns the framework into executable procedures by specifying the Gaussian reference measure and the training/inference algorithms.

\subsection{Mixtures of Probability Paths}\label{subsec:mixture_paths}

In finite-dimensional space, a marginal probability path can be written as a mixture of conditional density paths:
\begin{equation}
    p_t(x)=\int p_t(x\mid z)\,q(z)\,dz,
\label{eq:finite_path_mixture}
\end{equation}
where $q$ is a distribution over the conditioning variable $z$.
However, in an infinite-dimensional separable Hilbert space $(\mathcal{F},\langle\cdot,\cdot\rangle_{\mathcal{F}})$,
there is no translation-invariant Lebesgue reference measure, so density-based formulations such as
Eq.~\eqref{eq:finite_path_mixture} are in general ill-defined.
We therefore formulate the mixture path directly at the level of probability measures and the weak continuity equation
(Eq.~\eqref{eq:ffm_continuity}).

\paragraph{Mixture of conditional measures}
We take the conditioning variable to be the target function $f\sim\nu$.
For each $f\in\mathcal{F}$, let $\mu_t^f$ be a conditional probability measure on $\mathcal{F}$.
Assume that for every Borel set $A\in\mathcal{B}(\mathcal{F})$, the map $f\mapsto\mu_t^f(A)$ is measurable.
The marginal (mixture) measure $\mu_t$ is defined by
\begin{equation}
    \mu_t(A)=\int_{\mathcal{F}} \mu_t^f(A)\,d\nu(f),
    \qquad \forall A\in\mathcal{B}(\mathcal{F}).
\label{eq:mix_measure}
\end{equation}
Equivalently, for any bounded measurable $h:\mathcal{F}\to\mathbb{R}$,
\begin{equation}
    \int_{\mathcal{F}} h(g)\,d\mu_t(g)
    =
    \int_{\mathcal{F}}\left(\int_{\mathcal{F}} h(g)\,d\mu_t^f(g)\right)d\nu(f).
\label{eq:mix_integral_identity}
\end{equation}

\paragraph{Aggregating conditional vector fields}
Let $u_t^f:\mathcal{F}\to\mathcal{F}$ be the conditional vector field generating $\mu_t^f$
(in the weak continuity equation sense).
We assume the square-integrability condition
\begin{equation}
    \int_{\mathcal{F}}\int_{\mathcal{F}} \|u_t^f(g)\|_{\mathcal{F}}^2\, d\mu_t^f(g)\, d\nu(f) < \infty.
\label{eq:cond_vector_L2}
\end{equation}
The marginal vector field $u_t$ is defined implicitly via its action on $\mu_t$:
\begin{equation}
    \int_{\mathcal{F}} \langle u_t(g), \xi(g)\rangle_{\mathcal{F}}\, d\mu_t(g)
    =
    \int_{\mathcal{F}}\left(\int_{\mathcal{F}}\langle u_t^f(g), \xi(g)\rangle_{\mathcal{F}}\, d\mu_t^f(g)\right)d\nu(f),
    \quad \forall\, \xi\in L^2(\mu_t;\mathcal{F}).
\label{eq:mix_vector_weak}
\end{equation}

\begin{corollary}[Existence and Uniqueness of the Marginal Vector Field]
\label{cor:existence_uniqueness}
Under \eqref{eq:cond_vector_L2}, there exists a unique $u_t\in L^2(\mu_t;\mathcal{F})$ (unique $\mu_t$-a.e.)
satisfying \eqref{eq:mix_vector_weak}.
\end{corollary}

\begin{remark}[Conditional expectation and Radon--Nikodym viewpoint]
Define the joint probability measure on $\mathcal{F}\times\mathcal{F}$ by
$\pi_t(df,dg):=\nu(df)\,\mu_t^f(dg)$, whose $g$-marginal is $\mu_t$.
Let $(F,G)\sim\pi_t$ and set $U:=u_t^{F}(G)\in L^2(\pi_t;\mathcal{F})$.
Then $u_t(G)$ can be identified with the Bochner conditional expectation $\mathbb{E}[U\mid G]$.
Equivalently, the $\mathcal{F}$-valued vector measure
\[
\mathbf{J}_t(B):=\int_{\mathcal{F}}\int_{B} u_t^f(g)\,\mu_t^f(dg)\,\nu(df),
\qquad \forall B\in\mathcal{B}(\mathcal{F}),
\]
satisfies $\mathbf{J}_t\ll \mu_t$ under \eqref{eq:cond_vector_L2}, and $u_t=d\mathbf{J}_t/d\mu_t$ in $L^2(\mu_t;\mathcal{F})$.
Moreover, if $\mu_t^f\ll \mu_t$ for $\nu$-a.e.\ $f$, then \eqref{eq:mix_vector_weak} implies the pointwise aggregation formula
\begin{equation}
    u_t(g)=\int_{\mathcal{F}} u_t^f(g)\,\frac{d\mu_t^f}{d\mu_t}(g)\,d\nu(f),
    \qquad \mu_t\text{-a.e. }g,
\label{eq:mix_vector_pointwise}
\end{equation}
which matches the marginalization identity in Eq.~\eqref{eq:ffm_marginalization}.
\end{remark}

Having established the definitions of the marginal measure $\mu_t$ and the marginal vector field $u_t$, we now examine their dynamical consistency.
A fundamental property of the continuity equation in its weak form (Eq.~\eqref{eq:ffm_continuity}) is its linearity with respect to the signed measure.
Intuitively, since the marginal path is constructed as a superposition of conditional paths, and each conditional pair $(\mu_t^f, u_t^f)$ satisfies the continuity equation, the aggregated pair $(\mu_t, u_t)$ should preserve this property.
The following theorem rigorously formalizes this intuition, guaranteeing that the regression target $u_t$ defined in Eq.~\eqref{eq:mix_vector_weak} is indeed the correct vector field generating the data distribution.

\begin{theorem}[Mixture preserves the weak continuity equation]
\label{thm:mixture_continuity}
Assume that for $\nu$-a.e.\ $f\in\mathcal{F}$, the conditional pair $(\mu_t^f,u_t^f)$ satisfies the weak continuity equation
\eqref{eq:ffm_continuity}, namely
\begin{equation}
\int_{0}^{1}\!\int_{\mathcal{F}}
\Big(\partial_t \psi(g,t)+\big\langle u_t^f(g),\nabla_g\psi(g,t)\big\rangle_{\mathcal{F}}\Big)\,d\mu_t^f(g)\,dt =0,
\label{eq:cond_ffm_continuity}
\end{equation}
for all appropriate test functions $\psi:\mathcal{F}\times[0,1]\to\mathbb{R}$
(e.g.\ $\psi(\cdot,0)=\psi(\cdot,1)=0$ and $\psi,\partial_t\psi,\nabla_g\psi$ bounded),
and assume the measurability/integrability conditions needed for Fubini/Tonelli (e.g.\ \eqref{eq:cond_vector_L2} with bounded $\nabla_g\psi$).
Let $\mu_t$ be defined by \eqref{eq:mix_measure} and let $u_t$ be defined by \eqref{eq:mix_vector_weak}.
Then $(\mu_t,u_t)$ satisfies \eqref{eq:ffm_continuity}.
\end{theorem}

\subsection{Learning the Marginal Vector Field}\label{subsec:objective}

We are interested in the scenario where the conditional probability paths $\mu_t^f$ and conditional vector fields $u_t^f$
are known and have a simpler form that connects the source and target distributions, and we wish to recover the marginal vector field
$u_t$ that generates the mixture path $\mu_t$.
Directly computing $u_t(g)$ via \eqref{eq:mix_vector_pointwise} (or equivalently via a Radon--Nikodym derivative) is generally intractable.
Instead, we construct an unbiased stochastic objective for regressing a learned operator $u_\theta$ to $u_t$, generalizing the
finite-dimensional flow matching objective into an infinite functional dimension.

Let $u_\theta: [0,1] \times \mathcal{F} \to \mathcal{F}$ be a time-dependent vector field parametrized by a neural operator
(e.g., FNO) with weights $\theta$.
We define the ideal, albeit intractable, functional FM (FFM) objective with respect to the marginal measure $\mu_t$:
\begin{equation}
    \mathcal{L}_{\mathrm{FFM}}(\theta)
    := \mathbb{E}_{t \sim \mathcal{U}[0,1]} \int_{\mathcal{F}} \| u_\theta(t, g) - u_t(g) \|_{\mathcal{F}}^2 \, d\mu_t(g).
    \label{eq:loss_fm_infinite}
\end{equation}
Minimizing \eqref{eq:loss_fm_infinite} ensures that $u_\theta$ approximates the true marginal vector field $u_t$ in the
$L^2(\mu_t; \mathcal{F})$ norm. However, since $u_t$ is unknown, we cannot optimize \eqref{eq:loss_fm_infinite} directly.

To overcome this, we extend the conditional objective to an infinite-dimensional space, noted as functional conditional flow matching (FCFM), which relies only on the tractable conditional fields $u_t^f$:
\begin{equation}
    \mathcal{L}_{\mathrm{FCFM}}(\theta)
    := \mathbb{E}_{t \sim \mathcal{U}[0,1]} \int_{\mathcal{F}}
    \left( \int_{\mathcal{F}} \| u_\theta(t, g) - u_t^f(g) \|_{\mathcal{F}}^2 \, d\mu_t^f(g) \right) d\nu(f).
    \label{eq:loss_cfm_infinite}
\end{equation}
This objective is efficient to estimate stochastically by sampling $t \sim \mathcal{U}[0,1]$, data $f \sim \nu$, and points $g \sim \mu_t^f$
(e.g., $g = m_t^f + \sigma_t^f f_0$ under Gaussian conditional paths).

\begin{theorem}[Equivalence of FFM and FCFM objectives in $\mathcal{F}$]
\label{thm:objective_equivalence}
Assume \eqref{eq:cond_vector_L2} holds and, for a.e.\ $t\in[0,1]$, the model satisfies
$u_\theta(t,\cdot)\in L^2(\mu_t;\mathcal{F})$.
Then
\begin{equation}
\mathcal{L}_{\mathrm{FCFM}}(\theta)
=
\mathcal{L}_{\mathrm{FFM}}(\theta)
+\mathbb E_{t\sim\mathcal U[0,1]}
\Big[
C(t) - \int_{\mathcal F}\|u_t(g)\|_{\mathcal F}^2\,d\mu_t(g)
\Big],
\label{eq:L_equivalence}
\end{equation}
where
\begin{equation}
C(t):=\int_{\mathcal F}\int_{\mathcal F}\|u_t^f(g)\|_{\mathcal F}^2\,d\mu_t^f(g)\,d\nu(f),
\label{eq:C_t_def}
\end{equation}
which is finite by \eqref{eq:cond_vector_L2}. In particular, the difference between $\mathcal{L}_{\mathrm{FCFM}}(\theta)$ and
$\mathcal{L}_{\mathrm{FFM}}(\theta)$ is independent of $\theta$.
Consequently,
\begin{equation}
\nabla_\theta \mathcal{L}_{\mathrm{FCFM}}(\theta)=\nabla_\theta \mathcal{L}_{\mathrm{FFM}}(\theta),
\label{eq:grad_equivalence}
\end{equation}
under standard conditions that justify interchanging $\nabla_\theta$ and integration.
\end{theorem}

\subsection{Optimal Transport of Functional CFM}
\label{subsec:FOT-CFM}

Standard FFM typically assumes an \textit{independent coupling} between the source measure $\mu_0$ and the target measure $\nu$.
Mathematically, this implies that the joint distribution is simply the product measure $\pi_0 = \mu_0 \otimes \nu$.
While valid for generating the correct marginal distribution, this independent coupling leads to stochastic trajectories that frequently intersect, resulting in a marginal vector field with high curvature and complexity.
Numerically, integrating such a curved vector field requires small step sizes (which will result in a high number of function evaluations) to limit discretization error.

So in this section, we use OT to enforce deterministic, straight-line probability paths by approximating the 2-Wasserstein optimal coupling.
Our method consists of two steps: (1) solving the static optimal transport problem within a mini-batch to align source and target samples, and (2) constructing the displacement interpolation (geodesic paths) based on this alignment.

\paragraph{Mini-batch Optimal Transport Coupling}
Since solving the global optimal transport problem over the entire infinite-dimensional dataset is computationally intractable, we adopt a stochastic approximation using mini-batches.
Consider a mini-batch of source samples $\mathcal{B}_0 = \{f_0^{(i)}\}_{i=1}^B \sim \mu_0$ and target samples $\mathcal{B}_1 = \{f_1^{(j)}\}_{j=1}^B \sim \nu$, where $B$ is the batch size.
Let $S_B$ denote the set of all permutations of the indices $\{1, \dots, B\}$.
We aim to find an optimal permutation $\sigma^* \in S_B$ that minimizes the total transport cost within the batch.
Here, each $\sigma \in S_B$ represents a bijective mapping which assigns the $i$-th source sample to the $\sigma(i)$-th target sample.
The optimization problem is given by:
\begin{equation}
    \sigma^* = \mathop{\arg\min}_{\sigma \in S_B} \sum_{i=1}^B \| f_0^{(i)} - f_1^{(\sigma(i))} \|_{\mathcal{F}}^2.
    \label{eq:minibatch_ot}
\end{equation}
This is a linear assignment problem, which we solve exactly using the Hungarian algorithm (or linear sum assignment) with a complexity of $\mathcal{O}(B^3)$.
To formalize the stochastic approximation induced by Eq. \ref{eq:minibatch_ot}, define the empirical source and target measures
\[
\hat{\mu}_0^B := \frac{1}{B}\sum_{i=1}^B \delta_{f_0^{(i)}},
\qquad
\hat{\nu}^B := \frac{1}{B}\sum_{j=1}^B \delta_{f_1^{(j)}}.
\]
Then Eq. \ref{eq:minibatch_ot} is precisely the quadratic optimal transport problem between the empirical measures \(\hat{\mu}_0^B\) and \(\hat{\nu}^B\). The following result shows that the mini-batch OT coupling used in FOT-CFM is a statistically consistent approximation of the population OT problem in the separable Hilbert space $\mathcal{F}$.

\begin{theorem}[Consistency of mini-batch OT in $\mathcal{F}$]
\label{thm:minibatch_ot_consistency}
Assume $(\mathcal{F},\langle\cdot,\cdot\rangle_F)$ is a separable Hilbert space and
$\mu_0,\nu\in\mathcal P_2(\mathcal{F})$. For each batch size \(B\), let $f_0^{(1)},\dots,f_0^{(B)} \overset{\mathrm{i.i.d.}}{\sim} \mu_0$ and $f_1^{(1)},\dots,f_1^{(B)} \overset{\mathrm{i.i.d.}}{\sim} \nu$, and define the empirical measures 
\[
\hat{\mu}_0^B:=\frac{1}{B}\sum_{i=1}^B\delta_{f_0^{(i)}},
\qquad
\hat{\nu}^B:=\frac{1}{B}\sum_{j=1}^B\delta_{f_1^{(j)}}.
\]
Let $\hat{\pi}_B\in\Pi(\hat{\mu}_0^B,\hat{\nu}^B)$ be an optimal coupling for the quadratic cost
\[
\int_{\mathcal{F}\times \mathcal{F}}\|x-y\|_\mathcal{F}^2\,d\pi(x,y),
\]
and define the interpolation map
\[
T_t(x,y):=(1-t)x+ty,
\qquad t\in[0,1].
\]
Then
\begin{equation}
W_2(\hat{\mu}_0^B,\mu_0)\to 0,
\qquad
W_2(\hat{\nu}^B,\nu)\to 0,
\label{eq:empirical_measure_convergence}
\end{equation}
almost surely as $B\to\infty$, and every weak limit point $\bar{\pi}$ of $\{\hat{\pi}_B\}_{B\ge1}$ satisfies
\begin{equation}
\bar{\pi}\in\Pi(\mu_0,\nu),
\qquad
\int_{\mathcal{F}\times \mathcal{F}}\|x-y\|_\mathcal{F}^2\,d\bar{\pi}(x,y)=W_2^2(\mu_0,\nu).
\label{eq:empirical_plan_optimality}
\end{equation}
That is, every subsequential limit of the mini-batch OT couplings is an optimal coupling of the population OT problem. In particular, if the population quadratic OT problem admits a unique optimal coupling $\pi^\ast$, then
\begin{equation}
\hat{\pi}_B \rightharpoonup \pi^\ast,
\qquad
(T_t)_\#\hat{\pi}_B \rightharpoonup (T_t)_\#\pi^\ast,
\qquad \forall\, t\in[0,1],
\label{eq:minibatch_ot_geodesic_convergence}
\end{equation}
almost surely, where $(T_t)_\#\pi^\ast$ is the population displacement interpolation. Consequently, the straight-line paths induced by mini-batch OT provide statistically consistent approximations of the global Wasserstein geodesic.

Moreover, in the equal-weight empirical case, an optimal empirical coupling may be chosen in the form
\begin{equation}
\hat{\pi}_B = \frac{1}{B}\sum_{i=1}^B
\delta_{\bigl(f_0^{(i)},\,f_1^{(\sigma_B(i))}\bigr)},
\label{eq:empirical_permutation_coupling}
\end{equation}
where $\sigma_B\in S_B$ is a minimizer of the mini-batch assignment problem in Eq. \ref{eq:minibatch_ot}.
\end{theorem}

\paragraph{Constructing Paths $($Dynamic OT $)$}
Once the optimal pairs $(f_0^{(i)}, f_1^{(\sigma(i))})$ are established, we construct the conditional probability paths to follow the Wasserstein geodesics.
According to the theory of dynamic optimal transport (see Eq.~\eqref{eq:dynamic_ot}), the path minimizing the kinetic energy for the quadratic cost is the displacement interpolation:
\begin{equation}
    f_t^{(i)} = (1-t)f_0^{(i)} + t f_1^{(\sigma(i))}.
    \label{eq:geodesic_path}
\end{equation}
The corresponding conditional vector field $u_t^{(i)}(\cdot)$ is a constant velocity field pointing from source to target:
\begin{equation}
    u_t^{(i)}(f_t^{(i)}) = f_1^{(\sigma(i))} - f_0^{(i)}.
    \label{eq:ot_vector_field}
\end{equation}
Unlike the Variance Preserving (VP) paths used in diffusion models which follow curved trajectories, Eq.~\eqref{eq:geodesic_path} describes a strictly straight trajectory in the Hilbert space $\mathcal{F}$ with constant speed.
Crucially, because the OT coupling minimizes the total distance $\sum \|f_1^{(\sigma(i))} - f_0^{(i)}\|^2$, the resulting straight paths tend to be better aligned and empirically exhibit reduced curvature / fewer crossings.

\paragraph{FOT-CFM Training Objective}
By substituting the OT-aligned pairs and the geodesic vector field into the general CFM objective (Eq.~\eqref{eq:loss_cfm_infinite}), we obtain the specific loss function for FOT-CFM:
\begin{equation}
    \mathcal{L}_{\mathrm{FOT-CFM}}(\theta) = \mathbb{E}_{t, \mathcal{B}_0, \mathcal{B}_1} \left[ \frac{1}{B} \sum_{i=1}^B \| u_\theta(t, f_t^{(i)}) - (f_1^{(\sigma(i))} - f_0^{(i)}) \|_{\mathcal{F}}^2 \right],
\end{equation}
where $t \sim \mathcal{U}[0,1]$ and $f_t^{(i)}$ is the interpolated sample.
By learning to regress this OT-guided geodesic vector field, $u_\theta$ approximates the velocity field associated with the OT-aligned displacement interpolation. In view of Theorem \ref{thm:minibatch_ot_consistency}, this mini-batch construction is a consistent approximation of the corresponding OT geometry. During inference, this results in significantly straighter flow trajectories, allowing the ODE solver to traverse from noise to data with large steps while maintaining high generation fidelity.

\subsection{Algorithm}
\label{subsec:algorithm}

Since white noise is undefined in infinite-dimensional Hilbert spaces \cite{zhang2024functional}, FOT-CFM initializes the generative process using functions sampled from a well-defined reference Gaussian measure $\mu_0$ (e.g., a Gaussian Random Field with a specified covariance kernel).
The vector field $u_\theta$ is parameterized by a resolution-invariant Neural Operator (e.g., FNO), which takes the time coordinate $t$ and function state $f_t$ as inputs.
Based on the theoretical framework established in Section~\ref{subsec:FOT-CFM}, we detail the training procedure with Mini-batch Optimal Transport in Algorithm~\ref{alg:training} and the simulation-free sampling procedure in Algorithm~\ref{alg:sampling}.

\begin{algorithm}[H]
\caption{FOT-CFM Training with Mini-batch Optimal Transport}
\label{alg:training}
\begin{algorithmic}[1]
\Require Dataset $\mathcal{D} \sim \nu$, Batch size $B$, Gaussian measure $\mu_0$, Neural Operator $u_\theta$.
\State Initialize model parameters $\theta$.
\While{not converged}
    \State \textbf{1. Sample Batch:}
    \State Sample data batch $\mathcal{B}_1 = \{f_1^{(i)}\}_{i=1}^B \sim \mathcal{D}$.
    \State Sample noise batch $\mathcal{B}_0 = \{f_0^{(i)}\}_{i=1}^B \sim \mu_0$.

    \State \textbf{2. Optimal Transport Coupling:}
    \State Compute pairwise cost matrix $M \in \mathbb{R}^{B \times B}$ where $M_{ij} = \| f_0^{(i)} - f_1^{(j)} \|_{\mathcal{F}}^2$.
    \State Solve assignment problem: $\sigma^* = \mathop{\arg\min}_{\sigma \in S_B} \sum_{i} M_{i, \sigma(i)}$.
    \State Reorder data batch: $f_{1}^{(i)} \leftarrow f_{1}^{(\sigma^*(i))}$.

    \State \textbf{3. Construct Paths:}
    \State Sample time $t \sim \mathcal{U}[0, 1]$.
    \State Interpolate state: $f_t^{(i)} = (1-t)f_0^{(i)} + t f_1^{(i)}$.
    \State Compute target velocity: $v^{(i)} = f_1^{(i)} - f_0^{(i)}$.

    \State \textbf{4. Optimization Step:}
    \State Predict velocity: $\hat{v}^{(i)} = u_\theta(t, f_t^{(i)})$.
    \State Compute Loss: $\mathcal{L} = \frac{1}{B} \sum_{i=1}^B \| \hat{v}^{(i)} - v^{(i)} \|_{\mathcal{F}}^2$.
    \State Update $\theta$ using gradient descent $\nabla_\theta \mathcal{L}$.
\EndWhile
\end{algorithmic}
\end{algorithm}

\begin{algorithm}[H]
\caption{FOT-CFM Inference (Sampling)}
\label{alg:sampling}
\begin{algorithmic}[1]
\Require Trained Neural Operator $u_\theta$, Gaussian measure $\mu_0$, Number of ODE steps $N$ (or ODE solver tolerance).
\State \textbf{1. Initialization:}
\State Sample initial noise $f_0 \sim \mu_0$.
\State Define the ODE: $\frac{d f_t}{d t} = u_\theta(t, f_t)$.

\State \textbf{2. Numerical Integration (e.g., Euler / RK4):}
\State Set time grid $t_0=0, t_1=1/N, \dots, t_N=1$.
\For{$k = 0$ to $N-1$}
    \State \textit{// Euler Step}
    \State $v_k = u_\theta(t_k, f_{t_k})$
    \State $f_{t_{k+1}} = f_{t_k} + (t_{k+1} - t_k) \cdot v_k$
\EndFor

\State \textbf{Return} Generated function sample $f_1$.
\end{algorithmic}
\end{algorithm}

\section{Experiments and Results}\label{sec4}
To evaluate the effectiveness of our framework, we conduct experiments on three representative chaotic dynamical systems that exhibit rich multi-scale turbulent structures: the Navier-Stokes equations, Kolmogorov Flow, and the Hasegawa-Wakatani equations for complex plasma systems. These benchmarks, encompassing both widely-used public datasets \cite{kerrigan2022diffusion, li2022learning, rahman2022generative} and a more sophisticated plasma physics case \cite{castagna2024stylegan, greif2023physics}, provide a comprehensive testbed for our approach.
For all tasks, we adopt the Fourier Neural Operator (FNO) \cite{li2020fourier} as the backbone (see Appendix B for details) to model the velocity, which takes functions as both inputs and outputs; the models are then trained with Algorithm \ref{alg:training}.

\subsection{Evaluation Metrics}
\label{subsec:metrics}

To comprehensively evaluate the performance of FOT-CFM in generating high-fidelity functional data and its computational efficiency, we employ a suite of metrics covering physical consistency, distributional similarity, and inference speed.

\paragraph{1. Spectral Consistency Metrics}
In turbulence modeling, capturing the correct energy cascade across scales is fundamental. We evaluate spectral fidelity through two complementary approaches:

Radial Spectrum (RS). The radial energy spectrum $E(k)$ quantifies the energy distribution over wavenumber magnitudes $k = ||\mathbf{k}||$. For a function $f$, it is computed via the Fourier transform $\hat{f}$ by integrating over concentric shells. To assess the reconstruction of turbulent fluctuations, we calculate the Coefficient of Determination ($R^2$) and the Root Mean Squared Error (RMSE) between the logarithms of the generated and reference spectra ($\log E_{gen}(k)$ vs. $\log E_{ref}(k)$). Note that the zero-frequency mode ($k=0$) is excluded to focus on the inertial subrange and fine-scale structures.

Directional Spectrum (DS). To verify that the model captures directional flow structures (e.g., in Kolmogorov flow), we further compute the directional energy spectrum $E(k_x)$ and $E(k_y)$ by integrating the 2D spectrum along the $k_y$ and $k_x$ axes, respectively. We report the log-scale $R^2$ and RMSE for both $x$ and $y$ components. High $R^2$ and low RMSE in these metrics indicate that the generated fields preserve the correct physical anisotropy and lack spectral bias.

\paragraph{2. Density Consistency Metrics}
To assess the alignment of marginal value distributions between the real and generated ensembles, we evaluate the statistical fidelity of the physical quantities (e.g., velocity magnitudes). We flatten the high-dimensional function fields into scalar collections and estimate their continuous probability density functions (PDFs) using Gaussian Kernel Density Estimation (KDE). We then compare the estimated densities of the generated data against the ground truth by reporting:
\begin{itemize}
    \item Density RMSE: The Root Mean Squared Error between the PDFs, quantifying the absolute deviation in probability magnitudes.
    \item Density $R^2$: The Coefficient of Determination, measuring how well the shape of the generated distribution matches the reference.
\end{itemize}
High $R^2$ and low RMSE indicate that FOT-CFM accurately reproduces the global statistical properties and physical value ranges of the target system.

\paragraph{3. Computational Efficiency}
A core contribution of FOT-CFM is the linearization of generative paths via optimal transport. To quantify this, we report the number of function evaluations required by the ODE solver (e.g., dopri5, 4th order Runge–Kutta or Euler) to achieve a target error tolerance or visual quality. Lower NFE indicates straighter trajectories and higher efficiency.

\subsection{Kolmogorov Flow}
\label{subsec:Kolmogorov}

We evaluate the performance of FOT-CFM on the 2D Kolmogorov Flow, a classical benchmark for chaotic fluid dynamics governed by the incompressible Navier-Stokes equations with sinusoidal forcing. The system is defined on a torus $\mathbb{T}^2 = [0, 2\pi]^2$, following the dynamics:
\begin{equation}
    \partial_t u = -u \cdot \nabla u - \nabla p + \frac{1}{Re} \Delta u + sin(ny) \hat{x}, \quad \nabla \cdot u = 0,
\end{equation}
where $u$ is the velocity field, $p$ is the pressure, $Re>0$ is the Reynolds number.
We utilize the publicly available dataset provided by Li et al. \cite{li2022learning}, which consists of high-fidelity simulation snapshots. The data is discretized on a spatial grid of resolution $64 \times 64$. The goal is to learn the invariant measure (distribution) of the chaotic attractor from the training snapshots and generate new, physically consistent flow states.

We compare FOT-CFM against several state-of-the-art functional generative models: the Denoising Diffusion Operator (DDO) \cite{ddo}, Functional Flow Matching (FFM) \cite{ffm}, functional Denoising Diffusion Probabilistic Model (DDPM) \cite{kerrigan2022diffusion}, and Generative Adversarial Neural Operators (GANO) \cite{rahman2022generative}. We do not compare to non-functional methods, as we are primarily interested in developing discretization-invariant generative models. All noise was specified via a Gaussian process with a tuned Matérn kernel. For the sake of a fair comparison, we used the same architecture for all models, with the
exception of GANO which requires a generator and discriminator pair. For all models, we performed extensive hyperparameter tuning and report the best results.

\begin{table*}[htbp]
\centering
\caption{Comparison on Kolmogorov Flow. We evaluate physical fidelity using spectral metrics (radial and directional) and density consistency metrics, alongside computational efficiency (NFE). \textbf{Bold} indicates the best performance.}
\label{tab:kolmogorov_metrics}
\renewcommand{\arraystretch}{1.0}
\resizebox{0.9\textwidth}{!}{
\begin{tabular}{c|c|c|c|c|c|c|c}
\toprule
\multicolumn{3}{c|}{Metrics} & DDPM & FFM & DDO & GANO & FOT-CFM \\
\midrule
\multirow{9}{*}{NFE=5} & \multirow{2}{*}{KDE} & $R^2$ & 0.9897 & 0.9975 & 0.8833 & 0.8799 & \textbf{0.9982} \\
 & & RMSE & 0.0027 & 0.0013 & 0.0090 & 0.0092 & \textbf{0.0011} \\
 \cmidrule(l){2-8}
 & \multirow{2}{*}{RS} & $R^2$ & 0.3941 & 0.9946 & 0.5552 & 0.8008 & \textbf{0.9953}  \\
 & & RMSE & 1.0088 & 0.0949 & 0.8643 & 0.5784 & \textbf{0.0892}  \\
 \cmidrule(l){2-8}
 & \multirow{2}{*}{DS($kx$)} & $R^2$ & 0.0508 & 0.9913 &	0.2712 & 0.7023 & \textbf{0.9919}  \\
 & & RMSE & 1.0448 & 0.1000 & 0.9155 & 0.5851 & \textbf{0.0967}  \\
 \cmidrule(l){2-8}
 & \multirow{2}{*}{DS($ky$)} & $R^2$ & 0.0697 & 0.9871 & 0.2902 & 0.6660 & \textbf{0.9883} \\
 & & RMSE & 1.0191 & 0.1199 & 0.8901 & 0.6106 & \textbf{0.1145}  \\
\midrule
\multirow{9}{*}{NFE=10} & \multirow{2}{*}{KDE} & $R^2$ & 0.9779 & 0.9974 & 0.9837 & 0.8799 & \textbf{0.9982}  \\
 & & RMSE & 0.0039 & 0.0014 & 0.0034 & 0.0092 & \textbf{0.0011}  \\
 \cmidrule(l){2-8}
 & \multirow{2}{*}{RS} & $R^2$ & 0.5536 & 0.9947 & 0.9302 & 0.8008 & \textbf{0.9953}  \\
 & & RMSE & 0.8659 & 0.0940 & 0.3424 & 0.5784 & \textbf{0.0892} \\
 \cmidrule(l){2-8}
 & \multirow{2}{*}{DS($kx$)} & $R^2$ & 0.3006 & 0.9914 & 0.8792 & 0.7023 & \textbf{0.9919}  \\
 & & RMSE & 0.8968 & 0.0992 & 0.3727 & 0.5851 & \textbf{0.0965} \\
 \cmidrule(l){2-8}
 & \multirow{2}{*}{DS($ky$)} & $R^2$ & 0.3198 & 0.9876 & 0.8921 & 0.6660 & \textbf{0.9885}  \\
 & & RMSE & 0.8714 & 0.1178 & 0.3471 & 0.6106 & \textbf{0.1133}  \\
\midrule
\multirow{9}{*}{NFE=20} & \multirow{2}{*}{KDE} & $R^2$ & 0.8898 & 0.9974 & 0.9973 & 0.8799 & \textbf{0.9985}  \\
 & & RMSE & 0.0088 & 0.0013 & 0.0014 & 0.0092 & \textbf{0.0017}  \\
 \cmidrule(l){2-8}
 & \multirow{2}{*}{RS} & $R^2$ & 0.7204 & 0.9948 & 0.9848 & 0.8008 & \textbf{0.9953} \\
 & & RMSE & 0.6853 & 0.0938 & 0.1599 & 0.5784 & \textbf{0.0890} \\
 \cmidrule(l){2-8}
 & \multirow{2}{*}{DS($kx$)} & $R^2$ & 0.5633 & 0.9915 & 0.9711 & 0.7023 & \textbf{0.9919} \\
 & & RMSE & 0.7087 & 0.0991 & 0.1823 & 0.5851 & \textbf{0.0964} \\
 \cmidrule(l){2-8}
 & \multirow{2}{*}{DS($ky$)} & $R^2$ & 0.5800 & 0.9876 & 0.9776 & 0.6660 & \textbf{0.9885} \\
 & & RMSE & 0.6847 & 0.1176 & 0.1582 & 0.6106 & \textbf{0.1131} \\
\midrule
\multirow{9}{*}{NFE=100} & \multirow{2}{*}{KDE} & $R^2$ & 0.7459 & 0.9974 & \textbf{0.9995} & 0.8799 & 0.9987 \\
 & & RMSE & 0.0133 & 0.0013 & \textbf{0.0006} & 0.0092 & 0.0019 \\
 \cmidrule(l){2-8}
 & \multirow{2}{*}{RS} & $R^2$ & 0.9020 & 0.9948 & \textbf{0.9971} & 0.8008 & 0.9963   \\
 & & RMSE & 0.4057 & 0.0938 & \textbf{0.0702} & 0.5784 & 0.0894  \\
 \cmidrule(l){2-8}
 & \multirow{2}{*}{DS($kx$)} & $R^2$ & 0.8718 & 0.9915 & \textbf{0.9931} & 0.7023 & 0.9921 \\
 & & RMSE & 0.3840 & 0.0991 & \textbf{0.0893} & 0.5851 & 0.0924  \\
 \cmidrule(l){2-8}
 & \multirow{2}{*}{DS($ky$)} & $R^2$ & 0.8667 & 0.9876 & \textbf{0.9931} & 0.6660 & 0.9896 \\
 & & RMSE & 0.3857 & 0.1176 & \textbf{0.0880} & 0.6106 & 0.1012 \\
\bottomrule
\end{tabular}
}
\end{table*}

As summarized in Table \ref{tab:kolmogorov_metrics} and Fig. \ref{fig:kolmogorov}, FOT-CFM achieves the best overall spectral and statistical consistency under low inference budgets (NFE=5–20), while remaining competitive at higher NFE. For the isotropic energy spectrum, FOT-CFM attains the highest $R^2$ and the lowest RMSE at NFE=5, indicating that it captures the correct distribution of energy across spatial scales. Moreover, the directional spectrums ($k_x$ and $k_y$) show close agreement with the reference, suggesting that the anisotropy induced by the sinusoidal forcing is well preserved; in contrast, several baselines exhibit noticeable high-wavenumber deviations, shown as Fig. \ref{fig:kolmogorov}. The KDE metric further confirms that the generated vorticity values follow the reference statistics, reducing non-physical generations. Due to the global optimal transport coupling, FOT-CFM learns straighter generative trajectories. As a result, it achieves this fidelity with fewer function evaluations.

\begin{figure}[h]
\centering
\includegraphics[width=1\textwidth]{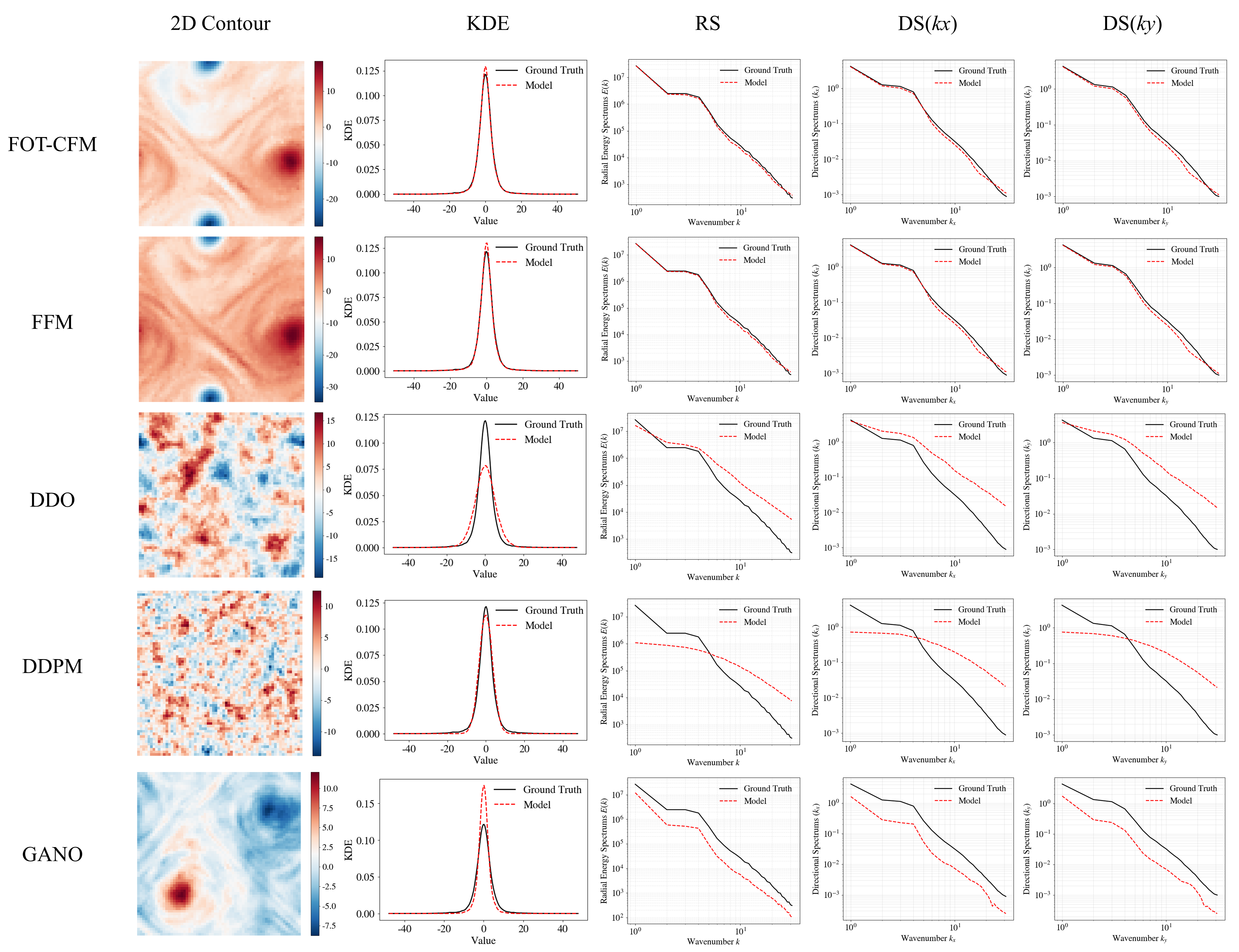}
\caption{Comparison of generative models on the 2D Kolmogorov Flow. Each row presents the results of a specific model (FOT-CFM, FFM, DDO, DDPM, and GANO) at a fixed inference budget (NFE=5).}\label{fig:kolmogorov}
\end{figure}

\subsection{Navier-Stokes Equations}
\label{subsec:NS}

To further validate the scalability and robustness of FOT-CFM, we consider the 2D incompressible Navier-Stokes equations. Unlike the forced Kolmogorov flow, this experiment focuses on the model's ability to represent the evolution of multi-scale vortices without continuous energy injection. The governing equations are formulated in terms of vorticity $\omega = \nabla \times \mathbf{u}$:
\begin{equation}
    \partial_t \omega + \mathbf{u} \cdot \nabla \omega = \nu \Delta \omega, \quad \nabla \cdot \mathbf{u} = 0,
\end{equation}
where $\nu$ is the kinematic viscosity. We use the dataset provided by Li et al. \cite{li2020fourier}, consisting of trajectory snapshots. The spatial resolution is $64 \times 64$, and we aim to generate diverse, physically valid flow states that conform to the target distribution of the turbulent attractor.

We maintain consistency with the previous experiment by comparing FOT-CFM against DDPM, FFM, DDO, and GANO. Evaluation is performed across three dimensions: Density RMSE and $R^2$ via Gaussian KDE, radial spectrum and directional ($k_x, k_y$) spectrums, number of function evaluations required for valid generations.

\begin{table*}[htbp]
\centering
\caption{Comparison on Navier-Stokes Equations. We evaluate physical fidelity using spectral metrics (radial and directional) and density consistency metrics, alongside computational efficiency (NFE). \textbf{Bold} indicates the best performance.}
\label{tab:ns_metrics}
\renewcommand{\arraystretch}{1.0}
\resizebox{0.9\textwidth}{!}{
\begin{tabular}{c|c|c|c|c|c|c|c}
\toprule
\multicolumn{3}{c|}{Metrics} & DDPM & FFM & DDO & GANO & FOT-CFM \\
\midrule
\multirow{9}{*}{NFE=5} & \multirow{2}{*}{KDE} & $R^2$ & 0.8848 & 0.9892 & 0.9412 & 0.9593 & \textbf{0.9949} \\
 & & RMSE & 0.0283 & 0.0087 & 0.0201 & 0.0168 & \textbf{0.0059} \\
 \cmidrule(l){2-8}
 & \multirow{2}{*}{RS} & $R^2$ & 0.1637 & 0.9536 & 0.4474 & 0.9149 & \textbf{0.9767} \\
 & & RMSE & 2.0988 & 0.2817 & 1.2896 & 0.5062 & \textbf{0.2649}  \\
 \cmidrule(l){2-8}
 & \multirow{2}{*}{DS($kx$)} & $R^2$ & 0.0464 & 0.8326 & 0.1047 & 0.6609 & \textbf{0.8910}  \\
 & & RMSE & 2.4661 & 0.5906 & 1.6312 & 1.0039 & \textbf{0.5692} \\
 \cmidrule(l){2-8}
 & \multirow{2}{*}{DS($ky$)} & $R^2$ & 0.0985 & 0.9129 & 0.1813 & 0.7195 & \textbf{0.9294}  \\
 & & RMSE & 2.3413 & 0.4904 & 1.5036 & 0.8802 & \textbf{0.4218} \\
\midrule
\multirow{9}{*}{NFE=10} & \multirow{2}{*}{KDE} & $R^2$ & 0.6965 & 0.9860 & 0.9516 & 0.9593 & \textbf{0.9891}  \\
 & & RMSE & 0.0460 & 0.0084 & 0.0184 & 0.0168 & \textbf{0.0067}  \\
 \cmidrule(l){2-8}
 & \multirow{2}{*}{RS} & $R^2$ & 0.2333 & 0.9797 & 0.9004 & 0.9149 & \textbf{0.9964}  \\
 & & RMSE & 1.9265 & 0.2472 & 0.5476 & 0.5062 & \textbf{0.1040}  \\
 \cmidrule(l){2-8}
 & \multirow{2}{*}{DS($kx$)} & $R^2$ & 0.1756 & 0.9024 & 0.7393 & 0.6609 & \textbf{0.9715}  \\
 & & RMSE & 2.2972 & 0.5386 & 0.8802 & 1.0039 & \textbf{0.2911}  \\
 \cmidrule(l){2-8}
 & \multirow{2}{*}{DS($ky$)} & $R^2$ & 0.1093 & 0.9273 & 0.7897 & 0.7195 & \textbf{0.9886}  \\
 & & RMSE & 2.1726 & 0.4482 & 0.7620 & 0.8802 & \textbf{0.1773}  \\
\midrule
\multirow{9}{*}{NFE=20} & \multirow{2}{*}{KDE} & $R^2$ & 0.4179 & \textbf{0.9941} & 0.9419 & 0.9593 & 0.9892   \\
 & & RMSE & 0.0636 & \textbf{0.0064} & 0.0201 & 0.0168 & 0.0086  \\
 \cmidrule(l){2-8}
 & \multirow{2}{*}{RS} & $R^2$ & 0.5747 & 0.9798 & 0.9611 & 0.9149 & \textbf{0.9829} \\
 & & RMSE & 1.6687 & 0.2464 & 0.3423 & 0.5062 & \textbf{0.2271}  \\
 \cmidrule(l){2-8}
 & \multirow{2}{*}{DS($kx$)} & $R^2$ & 0.4036 & 0.9028 & 0.8525 & 0.6609 & \textbf{0.9827}  \\
 & & RMSE & 2.0424 & 0.5374 & 0.6621 & 1.0039 & \textbf{0.3094} \\
 \cmidrule(l){2-8}
 & \multirow{2}{*}{DS($ky$)} & $R^2$ & 0.3333 & 0.9276 & 0.8871 & 0.7195 & \textbf{0.9752}  \\
 & & RMSE & 1.9189 & 0.4473 & 0.5584 & 0.8802 & \textbf{0.3232} \\
\midrule
\multirow{9}{*}{NFE=100} & \multirow{2}{*}{KDE} & $R^2$ & 0.7390 & \textbf{0.9943} & 0.9546 & 0.9593 & 0.9900 \\
 & & RMSE & 0.0426 & \textbf{0.0064} & 0.0178 & 0.0168 & 0.0083  \\
 \cmidrule(l){2-8}
 & \multirow{2}{*}{RS} & $R^2$ & 0.7890 & 0.9799 & 0.9773 & 0.9149 & \textbf{0.9932}  \\
 & & RMSE & 0.7969 & 0.2462 & 0.2613 & 0.5062 & \textbf{0.1432}  \\
 \cmidrule(l){2-8}
 & \multirow{2}{*}{DS($kx$)} & $R^2$ & 0.5398 & 0.9029 & 0.8911 & 0.6609 & \textbf{0.9835}  \\
 & & RMSE & 1.1694 & 0.5371 & 0.5689 & 1.0039 & \textbf{0.2216}  \\
 \cmidrule(l){2-8}
 & \multirow{2}{*}{DS($ky$)} & $R^2$ & 0.6026 & 0.9276 & 0.9152 & 0.7195 & \textbf{0.9927} \\
 & & RMSE & 1.0477 & 0.4470 & 0.4841 & 0.8802 & \textbf{0.1419} \\
\bottomrule
\end{tabular}
}
\end{table*}

The quantitative results are summarized in Table \ref{tab:ns_metrics}. FOT-CFM can still provide higher spectral fidelity across all computational budgets, demonstrating a strong ability to preserve the structure of turbulence. In particular, the directional spectrum, which is highly sensitive to high-wavenumber content, clearly reveals the advantage of FOT-CFM in the low-NFE regime, where it substantially outperforms diffusion-based baselines as well as the GAN model. For the radial spectrum, FOT-CFM achieves RS $R^2=0.9767$ at NFE=5, indicating accurate recovery from the inertial range to the dissipation range with very few function evaluations. The visualizations in Fig.~\ref{fig:ns_plots} further corroborate these findings, showing that FOT-CFM reproduces key turbulent structures. At low NFE, it attains the smallest errors among the benchmark methods, indicating the effectiveness of the proposed globally optimal transport coupling in infinite-dimensional functional spaces. Although FFM becomes slightly better on the KDE metric at larger NFEs (e.g., NFE=20 and 100), FOT-CFM remains superior on all spectral metrics, especially the directional spectrum, which best indicates the physical consistency of turbulent structures.

\begin{figure}[h]
\centering
\includegraphics[width=1\textwidth]{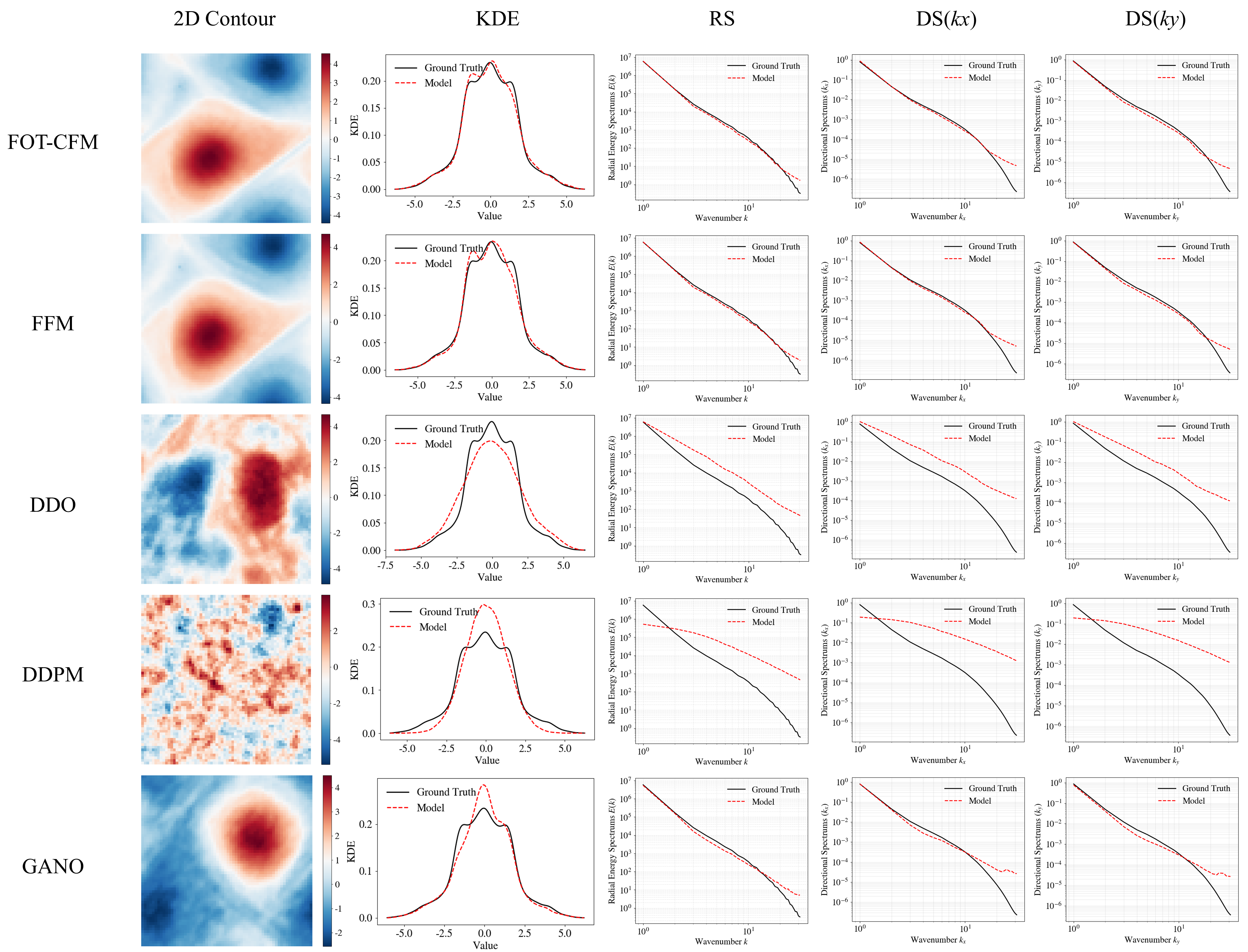}
\caption{Comparison of generative models on the 2D Navier-Stokes equations. Each row presents the results of a specific model (FOT-CFM, FFM, DDO, DDPM, and GANO) at a fixed inference budget (NFE=5).}\label{fig:ns_plots}
\end{figure}

Consistent with the Kolmogorov-flow results, FOT-CFM maintains high generation quality with significantly fewer integration steps. The straighter trajectories induced by functional optimal transport enable accurate sampling even with a simple ODE discretization, whereas diffusion-based approaches typically require more evaluations and more careful numerical treatment to mitigate trajectory curvature in infinite-dimensional functional spaces.

\subsection{Hasegawa-Wakatani Equations}
\label{subsec:HW}

To further evaluate the performance of FOT-CFM beyond the aforementioned public datasets, we consider a more challenging turbulence benchmark drawn from plasma physics. Specifically, we study the Hasegawa–Wakatani equations, which model resistive drift-wave turbulence in magnetized plasmas by coupling the evolution of the density field $n$ and the vorticity field $\omega$:

\begin{subequations}
\begin{align}
    \frac{\partial}{\partial t} n + [\phi,n] + \kappa \,\frac{\partial}{\partial y} \phi &= C \bigl(\phi - n\bigr) + D_{0}\,\nabla^2 n.  \label{hw_1}\\
    \frac{\partial}{\partial t} \omega + [\phi,\omega] &= C \bigl(\phi -  n\bigr) + D_{0}\,\nabla^2 \omega . \label{hw_2}
\end{align}
\end{subequations}
where $\phi$ is the electrostatic potential satisfying $\omega = \nabla^2 \phi$. The reference data is generated using the TOKAM2D \cite{tokam2d_gyselax,tokam2d2018,tokam2d2022} code on a $128 \times 128$ grid. 

A key advantage of function generation is its resolution-invariant formulation. To evaluate the model’s multiscale representational capability, we downsample the training data to $64 \times 64$ while performing inference at a higher resolution of $128 \times 128$. Because the model operates in a continuous functional space, it can produce high-resolution samples without being explicitly trained on $128 \times 128$ data, enabling zero-shot resolution scaling. This capability is particularly important for plasma simulations, where generating high-fidelity reference data is computationally costly. By leveraging the mesh-independent functional optimal transport path, FOT-CFM effectively interpolates the underlying physical fields while preserving fine-scale structures and overall structural integrity.

\begin{table*}[htbp]
\centering
\caption{Comparison on the density $n$ of Hasegawa-Wakatani Equations. We evaluate physical fidelity using spectral metrics (radial and directional) and density consistency metrics, alongside computational efficiency (NFE). \textbf{Bold} indicates the best performance.}
\label{tab:dens_metrics}
\renewcommand{\arraystretch}{1.0}
\resizebox{0.9\textwidth}{!}{
\begin{tabular}{c|c|c|c|c|c|c|c}
\toprule
\multicolumn{3}{c|}{Metrics} & DDPM & FFM & DDO & GANO & FOT-CFM \\
\midrule
\multirow{9}{*}{NFE=100} & \multirow{2}{*}{KDE} & $R^2$ & 0.2400 & 0.9856 & 0.9911 & 0.3412 & \textbf{0.9932} \\
 & & RMSE & 0.0629 & 0.0041 & 0.0046 & 0.0392 & \textbf{0.0038} \\
 \cmidrule(l){2-8}
 & \multirow{2}{*}{RS} & $R^2$ & 0.8735 & 0.9878 & 0.9811 & 0.5673 & \textbf{0.9912} \\
 & & RMSE & 0.5404 & 0.1377 & 0.1528 & 0.9995 & \textbf{0.1309}\\
 \cmidrule(l){2-8}
 & \multirow{2}{*}{DS($kx$)} & $R^2$ & 0.7947 & 0.9704 & 0.9713 & 0.2922 & \textbf{0.9814} \\
 & & RMSE & 0.5851 & 0.1708 & 0.1517 & 1.0864 & \textbf{0.1121} \\
 \cmidrule(l){2-8}
 & \multirow{2}{*}{DS($ky$)} & $R^2$ & 0.8187 & 0.9889 & 0.9832 & 0.3818 & \textbf{0.9891} \\
 & & RMSE & 0.5404 & 0.1338 & 0.1647 & 0.9978 & \textbf{0.1326} \\
\midrule
\multirow{9}{*}{NFE=500} & \multirow{2}{*}{KDE} & $R^2$ & 0.3898 & 0.9896 & 0.9913 & 0.3412 & \textbf{0.9951} \\
 & & RMSE & 0.0384 & 0.0042 & 0.0046 & 0.0392 & \textbf{0.0032} \\
 \cmidrule(l){2-8}
 & \multirow{2}{*}{RS} & $R^2$ & 0.9674 & 0.9902 & 0.9877 & 0.5673 & \textbf{0.9929} \\
 & & RMSE & 0.2742 & 0.1318 & 0.1685 & 0.9995 & \textbf{0.1298} \\
 \cmidrule(l){2-8}
 & \multirow{2}{*}{DS($kx$)} & $R^2$ & 0.9547 & 0.9728 & 0.9767 & 0.2922 & \textbf{0.9825}\\
 & & RMSE & 0.2749 & 0.1694 & 0.1450 & 1.0864 & \textbf{0.1010} \\
 \cmidrule(l){2-8}
 & \multirow{2}{*}{DS($ky$)} & $R^2$ & 0.9557 & 0.9889 & 0.9784 & 0.3818 & \textbf{0.9891} \\
 & & RMSE & 0.2671 & 0.1338 & 0.1864 & 0.9978 & \textbf{0.1326} \\
\midrule
\multirow{9}{*}{NFE=1000} & \multirow{2}{*}{KDE} & $R^2$ & 0.8746 & 0.9956 & 0.9924 & 0.3412 & \textbf{0.9957} \\
 & & RMSE & 0.0174 & 0.0032 & 0.0043 & 0.0392 & \textbf{0.0030} \\
 \cmidrule(l){2-8}
 & \multirow{2}{*}{RS} & $R^2$ & 0.9857 & 0.9928 & 0.9874 & 0.5673 & \textbf{0.9929}  \\
 & & RMSE & 0.1816 & 0.1289 & 0.1708 & 0.9995 & \textbf{0.1281} \\
 \cmidrule(l){2-8}
 & \multirow{2}{*}{DS($kx$)} & $R^2$ & 0.9813 & 0.9828 & \textbf{0.9871} & 0.2922 & 0.9855  \\
 & & RMSE & 0.1765 & 0.1694 & \textbf{0.1464} & 1.0864 & 0.1801 \\
 \cmidrule(l){2-8}
 & \multirow{2}{*}{DS($ky$)} & $R^2$ & 0.9806 & 0.9889 & 0.9778 & 0.3818 & \textbf{0.9891} \\
 & & RMSE & 0.1768 & 0.1338 & 0.1891 & 0.9978 & \textbf{0.1326} \\
\midrule
\multirow{9}{*}{NFE=1500} & \multirow{2}{*}{KDE} & $R^2$ & \textbf{0.9964} & 0.9956 & 0.9925 & 0.3412 & 0.9957 \\
 & & RMSE & \textbf{0.0029} & 0.0032 & 0.0043 & 0.0392 & 0.0030 \\
 \cmidrule(l){2-8}
 & \multirow{2}{*}{RS} & $R^2$ & \textbf{0.9935} & 0.9928 & 0.9873 & 0.5673 & 0.9929  \\
 & & RMSE & \textbf{0.1228} & 0.1289 & 0.1715 & 0.9995 & 0.1281 \\
 \cmidrule(l){2-8}
 & \multirow{2}{*}{DS($kx$)} & $R^2$ & \textbf{0.9894} & 0.9828 & 0.9871 & 0.2922 & 0.9825 \\
 & & RMSE & \textbf{0.1332} & 0.1694 & 0.1464 & 1.0864 & 0.1710 \\
 \cmidrule(l){2-8}
 & \multirow{2}{*}{DS($ky$)} & $R^2$ & \textbf{0.9909} & 0.9889 & 0.9776 & 0.3818 & 0.9891 \\
 & & RMSE & \textbf{0.1212} & 0.1338 & 0.1901 & 0.9978 & 0.1326 \\
\bottomrule
\end{tabular}
}
\end{table*}

\begin{table*}[htbp]
\centering
\caption{Comparison on the potential $\phi$ of Hasegawa-Wakatani Equations. We evaluate physical fidelity using spectral metrics (radial and directional) and density consistency metrics, alongside computational efficiency (NFE). \textbf{Bold} indicates the best performance.}
\label{tab:pot_metrics}
\renewcommand{\arraystretch}{1.0}
\resizebox{0.9\textwidth}{!}{
\begin{tabular}{c|c|c|c|c|c|c|c}
\toprule
\multicolumn{3}{c|}{Metrics} & DDPM & FFM & DDO & GANO & FOT-CFM \\
\midrule
\multirow{9}{*}{NFE=100} & \multirow{2}{*}{KDE} & $R^2$ & 0.1737 & 0.9905 & 0.9893 & 0.8976 & \textbf{0.9991} \\
 & & RMSE & 0.0709 & 0.0039 & 0.0054 & 0.0140 & \textbf{0.0016} \\
 \cmidrule(l){2-8}
 & \multirow{2}{*}{RS} & $R^2$ & 0.5937 & 0.9021 & 0.9220 & 0.3023 & \textbf{0.9928} \\
 & & RMSE & 1.3784 & 0.6765 & 0.6038 & 1.8062 & \textbf{0.1841} \\
 \cmidrule(l){2-8}
 & \multirow{2}{*}{DS($kx$)} & $R^2$ & 0.3781 & 0.8266 & 0.8620 & 0.2007 & \textbf{0.9511} \\
 & & RMSE & 1.5505 & 0.8186 & 0.7303 & 2.1544 & \textbf{0.3931} \\
 \cmidrule(l){2-8}
 & \multirow{2}{*}{DS($ky$)} & $R^2$ & 0.4157 & 0.8430 & 0.8707 & 0.7883 & \textbf{0.9531} \\
 & & RMSE & 1.4803 & 0.7675 & 0.6963 & 0.8910 & \textbf{0.4196} \\
\midrule
\multirow{9}{*}{NFE=500} & \multirow{2}{*}{KDE} & $R^2$ & 0.3669 & 0.9957 & 0.9898 & 0.8976 & \textbf{0.9991} \\
 & & RMSE & 0.0412 & 0.0033 & 0.0052 & 0.0140 & \textbf{0.0016} \\
\cmidrule(l){2-8}
& \multirow{2}{*}{RS} & $R^2$ & 0.9034 & 0.9021 & 0.9265 & 0.3023 & \textbf{0.9927} \\
 & & RMSE & 0.6722 & 0.6765 & 0.5861 & 1.8062 & \textbf{0.1843} \\
 \cmidrule(l){2-8}
 & \multirow{2}{*}{DS($kx$} & $R^2$ & 0.8322 & 0.8266 & 0.8709 & 0.2007 & \textbf{0.9600} \\
 & & RMSE & 0.8054 & 0.8186 & 0.7065 & 2.1544 & \textbf{0.3933} \\
 \cmidrule(l){2-8}
 & \multirow{2}{*}{DS($ky$)} & $R^2$ & 0.8476 & 0.8430 & 0.8763 & 0.7883 & \textbf{0.9530} \\
 & & RMSE & 0.7561 & 0.7675 & 0.6812 & 0.8910 & \textbf{0.4197} \\
\midrule
\multirow{9}{*}{NFE=1000} & \multirow{2}{*}{KDE} & $R^2$ & 0.8798 & 0.9957 & 0.9904 & 0.8976 & \textbf{0.9991} \\
 & & RMSE & 0.0180 & 0.0033 & 0.0051 & 0.0140 & \textbf{0.0016} \\
 \cmidrule(l){2-8}
 & \multirow{2}{*}{RS} & $R^2$ & 0.9266 & 0.9021 & 0.9273 & 0.3023 & \textbf{0.9927} \\
 & & RMSE & 0.5860 & 0.6765 & 0.5832 & 1.8062 & \textbf{0.1843} \\
 \cmidrule(l){2-8}
 & \multirow{2}{*}{DS($kx$)} & $R^2$ & 0.8672 & 0.8266 & 0.8697 & 0.2007 & \textbf{0.9600} \\
 & & RMSE & 0.7165 & 0.8186 & 0.7096 & 2.1544 & \textbf{0.3933} \\
 \cmidrule(l){2-8}
 & \multirow{2}{*}{DS($ky$)} & $R^2$ & 0.8802 & 0.8430 & 0.8790 & 0.7883 & \textbf{0.9530} \\
 & & RMSE & 0.6702 & 0.7675 & 0.6737 & 0.8910 & \textbf{0.4197} \\
\midrule
\multirow{9}{*}{NFE=1500} & \multirow{2}{*}{KDE} & $R^2$ & 0.9965 & 0.9957 & 0.9910 & 0.8976 & \textbf{0.9995} \\
 & & RMSE & 0.0031 & 0.0033 & 0.0049 & 0.0140 & \textbf{0.0012} \\
 \cmidrule(l){2-8}
 & \multirow{2}{*}{RS} & $R^2$ & 0.9229 & 0.9021 & 0.9276 & 0.3023 & \textbf{0.9936} \\
 & & RMSE & 0.6005 & 0.6765 & 0.5818 & 1.8062 & \textbf{0.1713} \\
 \cmidrule(l){2-8}
 & \multirow{2}{*}{DS($kx$)} & $R^2$ & 0.8605 & 0.8266 & 0.8698 & 0.2007 & \textbf{0.9679} \\
 & & RMSE & 0.7343 & 0.8186 & 0.7095 & 2.1544 & \textbf{0.3158} \\
 \cmidrule(l){2-8}
 & \multirow{2}{*}{DS($ky$)} & $R^2$ & 0.8738 & 0.8430 & 0.8799 & 0.7883 & \textbf{0.9663} \\
 & & RMSE & 0.6879 & 0.7675 & 0.6713 & 0.8910 & \textbf{0.3894} \\
\bottomrule
\end{tabular}
}
\end{table*}

The contour plots and spectral curves of the density $n$ and potential $\phi$ are compared in Fig.~\ref{fig:density} and Fig.~\ref{fig:potential}, respectively. Overall, FOT-CFM maintains good performance on this more complex and practically relevant turbulence problem. As shown in the contour plots, FOT-CFM generates coherent turbulent structures without noticeable fragmentation. Consistently, the spectral curves indicate close agreement with the reference results, confirming that the dominant low-frequency structures are faithfully captured.

\begin{figure}[h]
\centering
\includegraphics[width=1\textwidth]{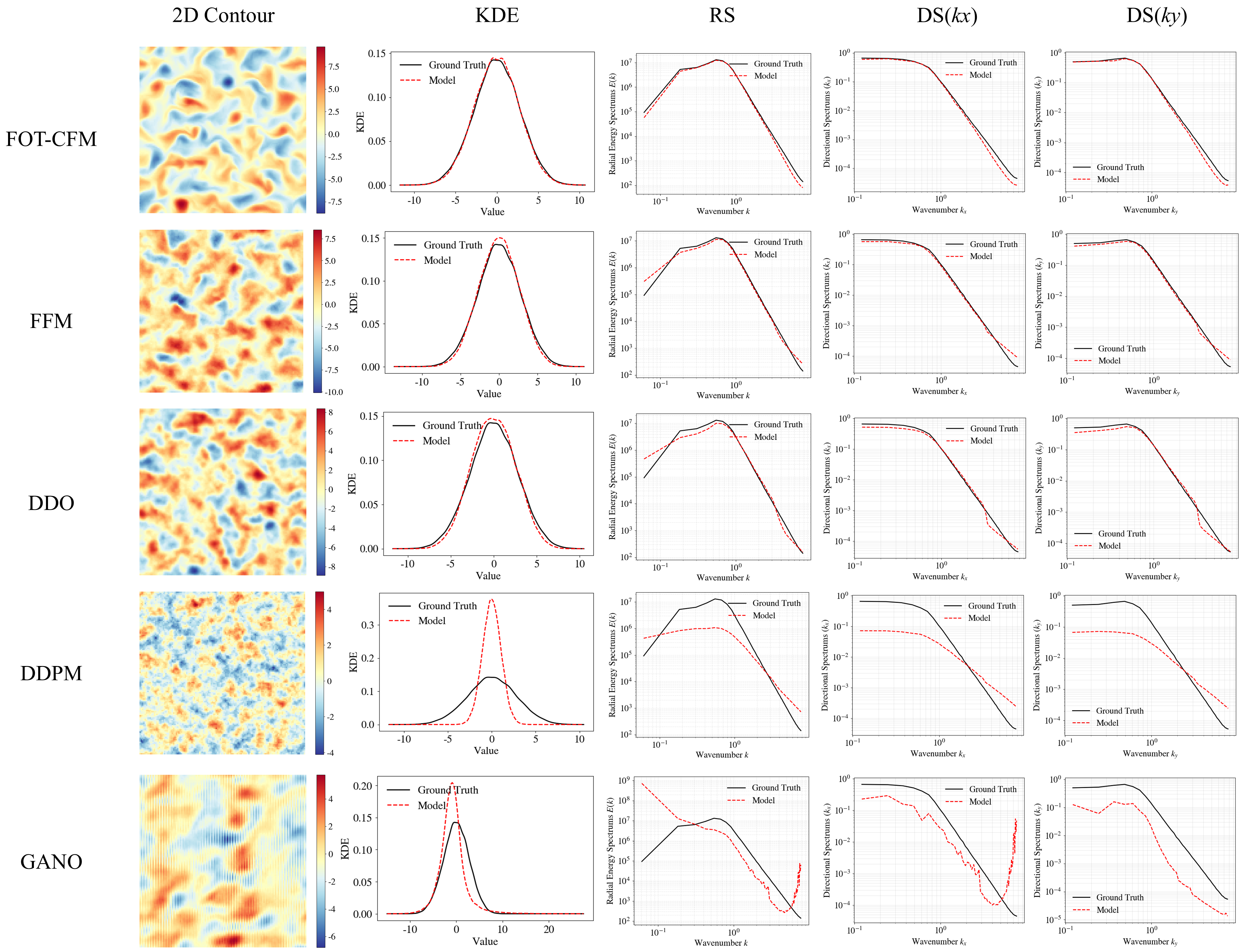}
\caption{Comparison of generative models on the density $n$ of 2D Hasegawa-Wakatani equations. Each row presents the results of a specific model (FOT-CFM, FFM, DDO, DDPM, and GANO) at a fixed inference budget (NFE=100).}\label{fig:density}
\end{figure}

\begin{figure}[h]
\centering
\includegraphics[width=1\textwidth]{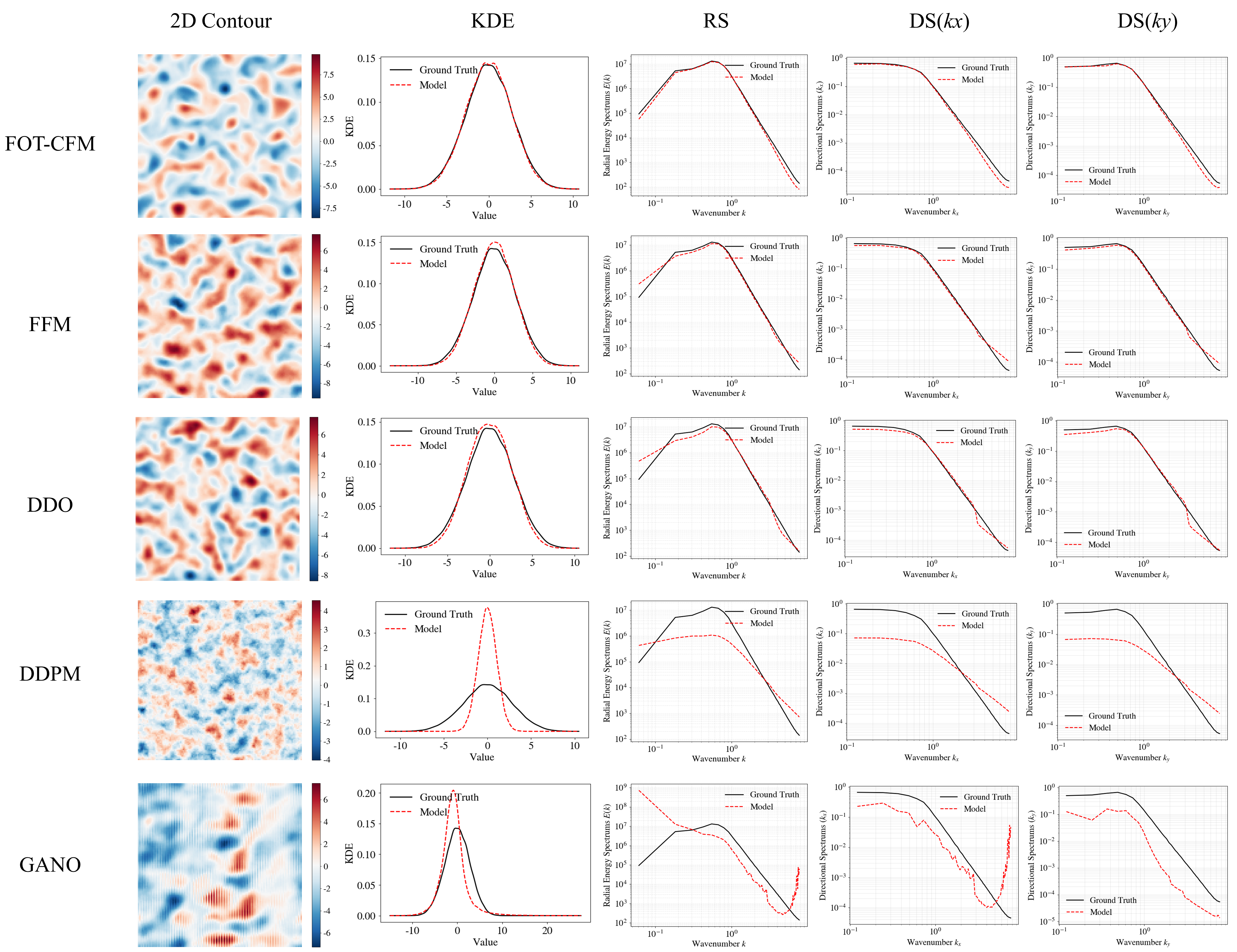}
\caption{Comparison of generative models on the potential $\phi$ of 2D Hasegawa-Wakatani equations. Each row presents the results of a specific model (FOT-CFM, FFM, DDO, DDPM, and GANO) at a fixed inference budget (NFE=100).}\label{fig:potential}
\end{figure}

\section{Conclusion}\label{sec5}
In this work, we presented Functional Optimal Transport Conditional Flow Matching (FOT-CFM), a generative framework for rapid and high-fidelity synthesis of complex scientific turbulence data. By constructing the probability path via functional optimal transport, our approach alleviates high-curvature generation trajectories in infinite-dimensional function spaces, enabling efficient sampling with substantially reduced computational cost.

Across experiments on 2D Kolmogorov flow, Navier–Stokes turbulence, and the Hasegawa–Wakatani system, we demonstrated several key advantages. First, FOT-CFM consistently outperforms state-of-the-art baselines, including DDPM, FFM, DDO, and GANO, in capturing multiscale turbulent structures, achieving strong spectral fidelity and accurately reproducing marginal density statistics. Second, by enforcing a globally optimal coupling, FOT-CFM reduces the inference budget, it produces high-quality samples with fewer NFEs without sacrificing physical consistency. Third, on the more complex and practically relevant TOKAM2D plasma turbulence dataset, FOT-CFM exhibits robust zero-shot scaling. Although trained only on $64 \times 64$ samples, it successfully generates physically consistent $128 \times 128$ density and potential fields, recovering fine-scale features beyond the training resolution.

Future work will extend FOT-CFM to 3D turbulence and investigate its integration with downstream applications, including uncertainty quantification and data-driven closure modeling for extreme-scale simulations.

\appendix

\section*{Appendix}
\addcontentsline{toc}{section}{Appendix}

\setcounter{section}{0}
\renewcommand{\thesection}{\Alph{section}}

\section{Proofs of Corollary and Theorem}\label{app:proofs_section}

\subsection{Proof of Corollary~\ref{cor:existence_uniqueness}}\label{app:proof_cor_exist_unique}
\begin{proof}
Fix $t\in[0,1]$.
Define a linear functional $L$ on $L^2(\mu_t;\mathcal F)$ by
\[
L(\xi)
:=
\int_{\mathcal{F}}\left(\int_{\mathcal{F}}\langle u_t^f(g), \xi(g)\rangle_{\mathcal{F}}\, d\mu_t^f(g)\right)d\nu(f).
\]
By Cauchy--Schwarz and \eqref{eq:mix_integral_identity},
\begin{align*}
|L(\xi)|
&\le
\int_{\mathcal F}
\left(\int_{\mathcal F}\|u_t^f(g)\|_{\mathcal F}^2\,d\mu_t^f(g)\right)^{1/2}
\left(\int_{\mathcal F}\|\xi(g)\|_{\mathcal F}^2\,d\mu_t^f(g)\right)^{1/2}
d\nu(f)\\
&\le
\left(\int_{\mathcal F}\int_{\mathcal F}\|u_t^f(g)\|_{\mathcal F}^2\,d\mu_t^f(g)\,d\nu(f)\right)^{1/2}
\left(\int_{\mathcal F}\int_{\mathcal F}\|\xi(g)\|_{\mathcal F}^2\,d\mu_t^f(g)\,d\nu(f)\right)^{1/2}\\
&=
\left(\int_{\mathcal F}\int_{\mathcal F}\|u_t^f(g)\|_{\mathcal F}^2\,d\mu_t^f(g)\,d\nu(f)\right)^{1/2}
\left(\int_{\mathcal F}\|\xi(g)\|_{\mathcal F}^2\,d\mu_t(g)\right)^{1/2},
\end{align*}
where the last equality used \eqref{eq:mix_integral_identity} with $h(g)=\|\xi(g)\|_{\mathcal F}^2$.
Thus $L$ is a bounded linear functional on the Hilbert space $L^2(\mu_t;\mathcal F)$.
By the Riesz representation theorem, there exists a unique $u_t\in L^2(\mu_t;\mathcal F)$ (unique $\mu_t$-a.e.)
such that
\[
L(\xi)=\int_{\mathcal F}\langle u_t(g),\xi(g)\rangle_{\mathcal F}\,d\mu_t(g),
\qquad \forall\,\xi\in L^2(\mu_t;\mathcal F),
\]
which is exactly \eqref{eq:mix_vector_weak}.
\end{proof}

\subsection{Proof of Theorem~\ref{thm:mixture_continuity}}\label{app:proof_thm_mixture_continuity}
\begin{proof}
Let $\psi$ be an appropriate test function as in Theorem~\ref{thm:mixture_continuity}.
For $\nu$-a.e.\ $f$, Eq.~\eqref{eq:cond_ffm_continuity} holds. Integrating both sides with respect to $d\nu(f)$ and applying
Fubini/Tonelli (justified by the assumed integrability and boundedness of $\nabla_g\psi$), we obtain
\begin{align*}
0
&=
\int_{\mathcal F}\int_0^1\int_{\mathcal F}
\Big(\partial_t \psi(g,t)+\langle u_t^f(g),\nabla_g\psi(g,t)\rangle_{\mathcal F}\Big)\,d\mu_t^f(g)\,dt\,d\nu(f)\\
&=
\int_0^1\left[
\int_{\mathcal F}\partial_t\psi(g,t)\,d\mu_t(g)
+
\int_{\mathcal F}\left(\int_{\mathcal F}\langle u_t^f(g),\nabla_g\psi(g,t)\rangle_{\mathcal F}\,d\mu_t^f(g)\right)d\nu(f)
\right]dt,
\end{align*}
where the first term used \eqref{eq:mix_integral_identity} with $h(g)=\partial_t\psi(g,t)$.
For the second term, fix $t$ and set $\xi_t(g):=\nabla_g\psi(g,t)$.
By the test-function regularity, $\xi_t\in L^2(\mu_t;\mathcal F)$, hence \eqref{eq:mix_vector_weak} yields
\[
\int_{\mathcal F}\left(\int_{\mathcal F}\langle u_t^f(g),\nabla_g\psi(g,t)\rangle_{\mathcal F}\,d\mu_t^f(g)\right)d\nu(f)
=
\int_{\mathcal F}\langle u_t(g),\nabla_g\psi(g,t)\rangle_{\mathcal F}\,d\mu_t(g).
\]
Substituting back proves \eqref{eq:ffm_continuity}.
\end{proof}

\subsection{Proof of Theorem~\ref{thm:objective_equivalence}}\label{app:proof_thm_objective_equivalence}
\begin{proof}
Fix $t\in[0,1]$ and define the time-slice objectives
\begin{align}
\mathcal{J}_{\mathrm{FM}}(\theta;t)
&:= \int_{\mathcal{F}} \| u_\theta(t, g) - u_t(g) \|_{\mathcal{F}}^2 \, d\mu_t(g),\\
\mathcal{J}_{\mathrm{CFM}}(\theta;t)
&:= \int_{\mathcal{F}}\left(\int_{\mathcal{F}} \| u_\theta(t, g) - u_t^f(g) \|_{\mathcal{F}}^2 \, d\mu_t^f(g)\right)d\nu(f).
\end{align}
Expanding $\mathcal{J}_{\mathrm{CFM}}(\theta;t)$ gives
\begin{align*}
\mathcal{J}_{\mathrm{CFM}}(\theta;t)
&=
\int_{\mathcal F}\int_{\mathcal F}
\Big(\|u_\theta(t,g)\|_{\mathcal F}^2
-2\langle u_\theta(t,g),u_t^f(g)\rangle_{\mathcal F}
+\|u_t^f(g)\|_{\mathcal F}^2\Big)\,d\mu_t^f(g)\,d\nu(f)\\
&=: T_1 -2T_2 + T_3.
\end{align*}
By \eqref{eq:mix_integral_identity} applied to $h(g)=\|u_\theta(t,g)\|_{\mathcal F}^2$ (integrable since
$u_\theta(t,\cdot)\in L^2(\mu_t)$), we have
\[
T_1=\int_{\mathcal F}\|u_\theta(t,g)\|_{\mathcal F}^2\,d\mu_t(g).
\]
For the cross term, apply \eqref{eq:mix_vector_weak} with the test vector field $\xi(g)=u_\theta(t,g)\in L^2(\mu_t;\mathcal F)$:
\[
T_2
=\int_{\mathcal F}\int_{\mathcal F}\langle u_\theta(t,g),u_t^f(g)\rangle_{\mathcal F}\,d\mu_t^f(g)\,d\nu(f)
=\int_{\mathcal F}\langle u_\theta(t,g),u_t(g)\rangle_{\mathcal F}\,d\mu_t(g).
\]
Finally, $T_3=C(t)$ as defined in \eqref{eq:C_t_def}, which is independent of $\theta$.
Therefore,
\[
\mathcal{J}_{\mathrm{CFM}}(\theta;t)
=
\int_{\mathcal F}\Big(\|u_\theta(t,g)\|_{\mathcal F}^2-2\langle u_\theta(t,g),u_t(g)\rangle_{\mathcal F}\Big)\,d\mu_t(g)
+ C(t).
\]
On the other hand,
\[
\mathcal{J}_{\mathrm{FM}}(\theta;t)
=
\int_{\mathcal F}\Big(\|u_\theta(t,g)\|_{\mathcal F}^2
-2\langle u_\theta(t,g),u_t(g)\rangle_{\mathcal F}
+\|u_t(g)\|_{\mathcal F}^2\Big)\,d\mu_t(g).
\]
Subtracting yields
\[
\mathcal{J}_{\mathrm{CFM}}(\theta;t)
=
\mathcal{J}_{\mathrm{FM}}(\theta;t)
+ C(t) - \int_{\mathcal F}\|u_t(g)\|_{\mathcal F}^2\,d\mu_t(g),
\]
and taking expectation over $t\sim\mathcal U[0,1]$ gives \eqref{eq:L_equivalence}.
Since the difference is independent of $\theta$, the gradients are identical under standard differentiation-under-the-integral conditions.
\end{proof}

\subsection{Proof of Theorem~\ref{thm:minibatch_ot_consistency}}\label{app:proof_thm_minibatch_ot_consistency}
\begin{proof}
For the empirical measures
\[
\hat{\mu}_0^B=\frac{1}{B}\sum_{i=1}^B\delta_{f_0^{(i)}},
\qquad
\hat{\nu}^B=\frac{1}{B}\sum_{j=1}^B\delta_{f_1^{(j)}},
\]
any coupling \(\pi\in\Pi(\hat{\mu}_0^B,\hat{\nu}^B)\) can be written as
\[
\pi=\sum_{i=1}^B\sum_{j=1}^B \gamma_{ij}\,\delta_{(f_0^{(i)},\,f_1^{(j)})},
\]
where \(\Gamma=(\gamma_{ij})\) is a nonnegative matrix satisfying
\[
\sum_{j=1}^B\gamma_{ij}=\frac{1}{B},
\qquad
\sum_{i=1}^B\gamma_{ij}=\frac{1}{B}.
\]
Equivalently, \(B\Gamma\) is doubly stochastic. Since the quadratic transport objective is linear in \(\Gamma\), an optimizer may be chosen at an extreme point of the Birkhoff polytope, hence at a permutation matrix. Therefore, an optimal empirical coupling may be taken in the form
\[
\hat{\pi}_B
=
\frac{1}{B}\sum_{i=1}^B
\delta_{\bigl(f_0^{(i)},\,f_1^{(\sigma_B(i))}\bigr)},
\]
for some \(\sigma_B\in S_B\), which yields \eqref{eq:empirical_permutation_coupling}.

Since \(\mathcal{F}\) is separable and \(\mu_0,\nu\in\mathcal P_2(\mathcal{F})\), the empirical measures \(\hat{\mu}_0^B\) and \(\hat{\nu}^B\) converge weakly almost surely to \(\mu_0\) and \(\nu\), respectively. Moreover, by the strong law of large numbers,
\[
\int_\mathcal{F} \|x\|_\mathcal{F}^2\,d\hat{\mu}_0^B(x)
=
\frac{1}{B}\sum_{i=1}^B \|f_0^{(i)}\|_\mathcal{F}^2
\longrightarrow
\int_\mathcal{F} \|x\|_\mathcal{F}^2\,d\mu_0(x),
\]
and similarly,
\[
\int_\mathcal{F} \|y\|_\mathcal{F}^2\,d\hat{\nu}^B(y)
=
\frac{1}{B}\sum_{j=1}^B \|f_1^{(j)}\|_\mathcal{F}^2
\longrightarrow
\int_\mathcal{F} \|y\|_\mathcal{F}^2\,d\nu(y),
\]
almost surely as \(B\to\infty\). Hence weak convergence together with convergence of second moments implies
\[
W_2(\hat{\mu}_0^B,\mu_0)\to 0,
\qquad
W_2(\hat{\nu}^B,\nu)\to 0,
\]
which proves \eqref{eq:empirical_measure_convergence}.

We prove that \(\{\hat{\pi}_B\}_{B\ge1}\) is tight in \(\mathcal P(\mathcal F\times\mathcal F)\).
Since \(\hat{\mu}_0^B\to\mu_0\) and \(\hat{\nu}^B\to\nu\) weakly on the Polish space \(\mathcal F\), the two families \(\{\hat{\mu}_0^B\}_{B\ge1}\) and \(\{\hat{\nu}^B\}_{B\ge1}\) are tight.
Hence, for any \(\varepsilon>0\), there exist compact sets
\(K_0,K_1\subset\mathcal F\) such that
\[
\hat{\mu}_0^B(K_0)\ge 1-\frac{\varepsilon}{2},
\qquad
\hat{\nu}^B(K_1)\ge 1-\frac{\varepsilon}{2},
\qquad \forall\, B\ge1.
\]
Therefore, for every \(B\),
\begin{align*}
\hat{\pi}_B\big((K_0\times K_1)^c\big)
&\le
\hat{\pi}_B(K_0^c\times \mathcal F)
+
\hat{\pi}_B(\mathcal F\times K_1^c)\\
&=
\hat{\mu}_0^B(K_0^c)+\hat{\nu}^B(K_1^c)
\le \varepsilon.
\end{align*}
Thus \(\{\hat{\pi}_B\}_{B\ge1}\) is tight in \(\mathcal P(\mathcal F\times\mathcal F)\).
Since \(\mathcal F\times\mathcal F\) is Polish, Prokhorov's theorem implies that
every subsequence of \(\{\hat{\pi}_B\}_{B\ge1}\) admits a further weakly convergent subsequence.

Let \(\{\hat{\pi}_{B_k}\}_{k\ge1}\) be an arbitrary subsequence.
By tightness, passing to a further subsequence if necessary, we may assume
\[
\hat{\pi}_{B_k}\rightharpoonup \bar{\pi}
\qquad\text{in }\mathcal P(\mathcal F\times\mathcal F).
\]

Since \(\hat{\pi}_{B_k}\) has marginals \(\hat{\mu}_0^{B_k}\) and \(\hat{\nu}^{B_k}\), for any bounded continuous \(\varphi:\mathcal{F}\to\mathbb R\),
\[
\int_\mathcal{F} \varphi(x)\,d(P_1)_\#\hat{\pi}_{B_k}(x)
=
\int_{\mathcal{F}\times \mathcal{F}}\varphi(x)\,d\hat{\pi}_{B_k}(x,y)
=
\int_\mathcal{F} \varphi(x)\,d\hat{\mu}_0^{B_k}(x)
\to
\int_\mathcal{F} \varphi(x)\,d\mu_0(x),
\]
and likewise
\[
\int_\mathcal{F} \varphi(y)\,d(P_2)_\#\hat{\pi}_{B_k}(y)
=
\int_{\mathcal{F}\times \mathcal{F}}\varphi(y)\,d\hat{\pi}_{B_k}(x,y)
=
\int_\mathcal{F} \varphi(y)\,d\hat{\nu}^{B_k}(y)
\to
\int_\mathcal{F} \varphi(y)\,d\nu(y).
\]
Therefore,
\[
(P_1)_\#\bar{\pi}=\mu_0,
\qquad
(P_2)_\#\bar{\pi}=\nu,
\]
so \(\bar{\pi}\in\Pi(\mu_0,\nu)\).

To prove optimality of \(\bar{\pi}\), define
\[
J_B:=\int_{\mathcal{F}\times \mathcal{F}}\|x-y\|_\mathcal{F}^2\,d\hat{\pi}_B(x,y)
=
W_2^2(\hat{\mu}_0^B,\hat{\nu}^B).
\]
By \eqref{eq:empirical_measure_convergence} and the continuity of \(W_2\),
\[
J_B\to W_2^2(\mu_0,\nu)
\qquad\text{almost surely.}
\]
Since the cost \((x,y)\mapsto \|x-y\|_\mathcal{F}^2\) is nonnegative and lower semicontinuous on \(\mathcal{F}\times \mathcal{F}\), the Portmanteau theorem gives
\[
\int_{\mathcal{F}\times \mathcal{F}}\|x-y\|_\mathcal{F}^2\,d\bar{\pi}(x,y)
\le
\liminf_{k\to\infty}
\int_{\mathcal{F}\times \mathcal{F}}\|x-y\|_\mathcal{F}^2\,d\hat{\pi}_{B_k}(x,y)
=
W_2^2(\mu_0,\nu).
\]
On the other hand, since \(\bar{\pi}\in\Pi(\mu_0,\nu)\),
\[
\int_{\mathcal{F}\times \mathcal{F}}\|x-y\|_\mathcal{F}^2\,d\bar{\pi}(x,y)
\ge
W_2^2(\mu_0,\nu).
\]
Thus
\[
\int_{\mathcal{F}\times \mathcal{F}}\|x-y\|_\mathcal{F}^2\,d\bar{\pi}(x,y)
=
W_2^2(\mu_0,\nu),
\]
which proves \eqref{eq:empirical_plan_optimality}.

Finally, assume the population quadratic OT problem admits a unique optimal coupling \(\pi^\ast\).
Let \(\{\hat{\pi}_{B_k}\}_{k\ge1}\) be an arbitrary subsequence.
By tightness, it admits a further weakly convergent subsequence, and by the previous argument
every such subsequential limit must equal \(\pi^\ast\).
Therefore every subsequence of \(\{\hat{\pi}_B\}_{B\ge1}\) has a further subsequence converging to \(\pi^\ast\),
which implies
\[
\hat{\pi}_B\rightharpoonup \pi^\ast
\qquad\text{almost surely.}
\]
For each fixed \(t\in[0,1]\), the interpolation map
\[
T_t(x,y):=(1-t)x+ty
\]
is continuous from \(\mathcal{F}\times \mathcal{F}\) into \(\mathcal{F}\). Therefore, by continuity of pushforward under weak convergence,
\[
(T_t)_\#\hat{\pi}_B \rightharpoonup (T_t)_\#\pi^\ast,
\qquad \forall\, t\in[0,1],
\]
which proves \eqref{eq:minibatch_ot_geodesic_convergence}.
\end{proof}

\section{Experiment Details}\label{app:others}
For FOT-CFM, FFM, DDPM and DDO, the architecture used is the FNO implemented in the
neuraloperator package \cite{li2020fourier, kovachki2023neural}. For GANO, we directly use the FNO-based model architectures for both the discriminator and generator implemented by Rahman et al. \cite{rahman2022generative}. Each model experimented with relies on noise sampled from a Gaussian measure. In current work, we consider a mean-zero Gaussian process (GP) parametrized by a Matérn kernel with $\nu=0.5$, which follows the setting of \cite{ffm}. The kernel parameters, including the variance and the length scale, are fine-tuned via grid search. The model-specific hyperparameters are directly adopted from \cite{ffm}. All models are implemented using PyTorch 2.2.1 \cite{paszke2019pytorch} and trained on an NVIDIA A100 GPU using the Adam optimizer \cite{kingma2014adam}.

\begin{itemize}
    \item \textbf{Kolmogorov Flow} This dataset consists of Kolmogorov flow solutions at a resolution of $64\times64$. To improve training efficiency, we randomly selected 10,000 samples from the dataset \cite{li2022learning} for training. For FOT-CFM, FFM, DDPM, and DNO, we use four Fourier layers with 32 modes, 64 hidden channels, 256 lifting channels, and 256 projection channels, together with the GeLU activation function \cite{hendrycks2016gaussian}. For GANO, we also use 32 modes, but reduce the number of hidden channels to 32 due to memory constraints. All models are trained for 500 epochs with a batch size of 128. We use the Adam optimizer with an initial learning rate of $1\times10^{-4}$. The learning rate follows a two-stage warmup plus cosine-annealing schedule: during the first $10\%$ of training epochs, a linear warmup increases the learning rate from $1\times10^{-10}$ (i.e., $10^{-6}$ times the base learning rate) to $1\times10^{-4}$; during the remaining $90\%$ of epochs, the learning rate is smoothly decayed via cosine annealing, with a minimum learning rate of $1\times10^{-6}$. This schedule improves optimization stability in the early stage and promotes smoother convergence in the later stage of training.

    \item \textbf{Navier-Stokes Equations} This dataset is adopted from \cite{li2020fourier}, which contains solutions of the Navier-Stokes equations. For training efficiency, we randomly sample 20,000 frames from the original dataset. The model architecture settings are the same as those used for the Kolmogorov flow experiments. We also use the same two-stage warmup plus cosine-annealing learning rate schedule, but set the initial learning rate to $5\times10^{-4}$. All models are trained for 500 epochs with a batch size of 128.

    \item \textbf{Hasegawa-Wakatani Equations} This dataset is generated using the official TOKAM2D repository\footnote{\url{https://github.com/gyselax/tokam2d}}, which is used for plasma turbulence research. In the governing equations, the adiabatic coefficient is set to 1, and the dissipation coefficient is set to 0.01. The domain lengths in both the $x$ and $y$ directions (normalized by the reference Larmor radius) are 51.5. The original simulation resolution is $128\times128$, and the data are downsampled to $64\times64$ for training in order to verify the resolution-invariant capability. For this case, we use an 8-layer FNO backbone, which provides a larger receptive field and stronger representation capacity for the more complex plasma turbulence dynamics. The architecture retains 64 Fourier modes with a hidden width of 64 channels. In total, 18,000 frames are used for training. All models are trained for 1000 epochs, and the initial learning rate is set to $5\times10^{-5}$. We use the same two-stage warmup plus cosine-annealing learning rate schedule as in the previous experiments.
\end{itemize}

\section*{Acknowledgement}
The authors also acknowledge the support from the National Research Foundation, Singapore. The authors would like to acknowledge the SAFE team for providing access to the TOKAM2D code, which was essential for the numerical simulations carried out in this study


\end{document}